\documentclass[10pt]{IEEEtran}

\usepackage{amsmath,graphicx}

\usepackage{lipsum}
\usepackage{fancyhdr}

\usepackage{amssymb,float,arydshln,color}
\usepackage{psfrag,setspace,wrapfig,subfigure}
\usepackage[latin1]{inputenc}
\usepackage{dsfont}
\usepackage[lined,ruled,commentsnumbered]{algorithm2e}
\usepackage{epsfig}
\usepackage{epstopdf}
\usepackage{graphicx}
\usepackage{amsthm}

\PassOptionsToPackage{bookmarks=false}{hyperref}

\usepackage{hyperref}
\usepackage{amsfonts}
\usepackage{cite}
\usepackage{url}

\usepackage{hhline}
\usepackage[table]{xcolor}
\usepackage{multirow}
\usepackage{tabu}
\allowdisplaybreaks

\newtheorem{lemma}{{Lemma}}
\newtheorem{assumption}{{ Assumption}}
\newtheorem{theorem}{{Theorem}}
\newtheorem{corollary}{{Corollary}}

\def\tran{^{\mathsf{T}}}

\DeclareMathOperator*{\argmin}{arg\,min}
\newcommand{\bp}{ \begin{proof}}
	\newcommand{\ep}{\end{proof} }

\newcommand{\Ex}{\mathbb{E}\hspace{0.05cm}}

\newcommand{\bm}[1]{\mbox{\boldmath $#1$}}

\newcommand{\be}{\begin{equation}}
\newcommand{\ee}{\end{equation}}
\newcommand{\bqq}{\begin{eqnarray}}
\newcommand{\eqq}{\end{eqnarray}}
\newcommand{\bal}{\begin{align}}
\newcommand{\eal}{\end{align}}
\newcommand{\bqn}{\begin{eqnarray*}}
	\newcommand{\eqn}{\end{eqnarray*}}
\newcommand{\nn}{\nonumber}
\newcommand{\ba}{\left[ \begin{array}}
	\newcommand{\ea}{\\ \end{array} \right]}
\newcommand{\qd}{\hfill{$\blacksquare$}}
\newcommand{\define}{\;\stackrel{\Delta}{=}\;}

\newcommand{\Tr}{\mbox{\rm {\small Tr}}}

\def\bsigma  {{\boldsymbol \sigma}}

\def\Pr {{\mathbb{P}}}

\def\n{{\boldsymbol{n}}}

\def\s{{\boldsymbol{s}}}

\def\u{{\boldsymbol{u}}}

\def\w{{\boldsymbol{w}}}

\newcommand{\cK}{{\mathcal{K}}}

\newcommand{\cN}{{\mathcal{N}}}

\newcommand{\cU}{{\mathcal{U}}}

\newcommand{\tw}{\widetilde{\boldsymbol{w}}}

\newcommand{\grad}{{\nabla}}

\def\bE{\mathbb{E}}

\newcommand{\eq}[1]{\begin{align}#1\end{align}}
\newcommand{\beqn}{\begin{eqnarray}}
\newcommand{\eeqn}{\end{eqnarray}}

\newcommand{\nnb}{\nonumber \\}

\def\real{{\mathbb{R}}}

\def\Zint{{\mathchoice{\setbox1=\hbox{\sf Z}\copy1\kern-.75\wd1\box1}
		{\setbox1=\hbox{\sf Z}\copy1\kern-.75\wd1\box1}
		{\setbox1=\hbox{\scriptsize\sf Z}\copy1\kern-.75\wd1\box1}
		{\setbox1=\hbox{\scriptsize\sf Z}\copy1\kern-.75\wd1\box1}}}

\title{Stochastic Learning under Random Reshuffling {\color{black}with Constant Step-sizes}}
\author{\IEEEauthorblockN{ Bicheng Ying\IEEEauthorrefmark{1}\IEEEauthorrefmark{2}, Kun Yuan\IEEEauthorrefmark{1}\IEEEauthorrefmark{2}, Stefan Vlaski\IEEEauthorrefmark{1}\IEEEauthorrefmark{2},  and Ali H. Sayed\IEEEauthorrefmark{2} \\}
	\vspace{0.4cm}

	\thanks{{This work was supported in part by NSF grants CCF-1524250 and ECCS-1407712. Emails:\;\{ybc, kunyuan, svlaski\}@ucla.edu and ali.sayed@epfl.ch}. 
	A short conference version of this work was presented in  \cite{ying2017rr}.}
	\IEEEauthorblockA{\IEEEauthorrefmark{1}Department of Electrical Engineering, University of California, Los Angeles}\\
	\IEEEauthorblockA{\IEEEauthorrefmark{2}School of Engineering,
		\'Ecole Polytechnique F\'ed\'erale de Lausanne, Switzerland}\vspace{-0.6cm}
}

\begin{document}
	%
	\maketitle

	\begin{abstract}
	In empirical risk optimization, it has been observed that stochastic gradient  implementations that rely on random reshuffling of the data achieve better performance than implementations that rely on sampling the data uniformly. Recent works have pursued justifications for this behavior by examining the convergence rate of the learning process under diminishing step-sizes. This work focuses on the constant step-size case {\color{black}and strongly convex loss functions}. In this case, convergence is guaranteed to a small neighborhood of the optimizer albeit at a linear rate. The analysis establishes analytically that random reshuffling outperforms uniform sampling by showing explicitly  that iterates approach a smaller neighborhood of size $O(\mu^2)$ around the minimizer rather than $O(\mu)$. Furthermore, we derive an analytical expression for the steady-state mean-square-error performance of the algorithm, which helps clarify in greater detail the differences between sampling with and without replacement. We also explain the periodic behavior that is observed in random reshuffling implementations. 
	\end{abstract}\vspace{-1mm}
	\begin{keywords}
	Random reshuffling, stochastic gradient descent, mean-square performance, convergence analysis, mean-square-error expression.\vspace{-3mm}
	\end{keywords}
	\setlength{\abovedisplayskip}{1.2mm}
	\setlength{\belowdisplayskip}{1.2mm}

\section{MOTIVATION}
We consider minimizing an empirical risk function $J(w)$, which is defined as the sample average of loss values over a possibly large but finite training set:\vspace{-2mm}
\eq{\label{prob-emp-into}
	w^\star \define \argmin_{w\in \real^M}\;  J(w) \define \frac{1}{N}\sum_{n=1}^{N}Q(w;x_n)\\[-6mm]\nn
}
where the $\{x_n\}_{n=1}^N$ denotes the training data samples and the loss function $Q(w;x_n)$ is assumed convex and differentiable. We assume the empirical risk $J(w)$ is strongly-convex which ensures that the minimizer, $w^{\star}$, is unique. Problems of the form (\ref{prob-emp-into}) are common
in many areas of machine learning including linear regression, logistic regression and their regularized versions.

When the size of the dataset $N$ is large, it is impractical to solve \eqref{prob-emp-into} directly via traditional gradient descent by evaluating the full gradient at every iteration. One simple, yet powerful remedy is to employ the stochastic gradient method (SGD) \cite{bertsekas1989parallel,polyak1992acceleration,polyak1987introduction,bottou2010large,bousquet2008tradeoffs,moulines2011non,zhang2004solving, needell2014stochastic}.
Rather than compute the full gradient $\grad_w J(w)$ over the entire data set, these algorithms pick one index $\boldsymbol{n}_i$ at random at every iteration, and employ $\grad_w Q(w;x_{\boldsymbol{n}_i})$ to approximate $\grad_w J(w)$. Specifically, at iteration $i$, the update for estimating the minimizer is of the form\cite{yuan2016stochastic}:\vspace{-1mm}
\eq{\label{alg:esgd-intr}
	\w_i = \w_{i-1} - \mu \grad_w Q(\w_{i-1};x_{\n_i}),
}
where $\mu$ is the step-size parameter.
Note that we are using boldface notation to refer to random variables. Traditionally, the index $\n_i$ is uniformly distributed over the discrete set $\{1,2,\ldots,N\}$.

It has been noted in the literature \cite{ bottou2009curiously,recht2012toward,gurbuzbalaban2015random, zhang2014note} that incorporating random reshuffling into the gradient descent implementation helps achieve better performance. More broadly than in the case of the pure SGD algorithm, it has also been observed that applying random reshuffling in variance-reduction algorithms, like SVRG\cite{Johnson2013}, SAGA\cite{defazio2014saga}, can accelerate the convergence speed\cite{de2016efficient,defazio2014finito,ying2017convergence, ying2018convergence}. The reshuffling technique has also been applied in distributed system to reduce the communication and computation cost\cite{lee2018speeding}.

In random reshuffling implementations, the data points are no longer picked independently and uniformly at random. Instead, the gradient descent algorithm is run multiple times over the data where each run is indexed by $k \geq 1$ and is referred to as an epoch. For each epoch, the original data is first reshuffled and then passed over in order. In this manner, the $i$-th sample of epoch $k$ can be viewed as ${\bsigma}^{k}(i)$, where the symbol $\bsigma$ represents a uniform random permutation of the indices. We can then express the random reshuffling algorithm for the $k-$th epoch in the following manner:
\be
\w_i^k = \w_{i-1}^k - \mu\grad_w Q(\w_{i-1}^k;x_{\bsigma^k(i)}),\;\;\;\; i = 1, \ldots, N \label{alg.rr}
\ee
with the boundary condition:
\be
\w^k_0= \w^{k-1}_N\label{bound.condtion}
\ee
In other words, the initial condition for epoch $k$ is the last iterate from epoch $k-1$. The boldface notation for the symbols $\w$ and $\bm{\sigma}$ in (\ref{alg.rr}) emphasizes the random nature of these variables due to the randomness in the permutation operation. While the samples over one epoch are no longer picked independently from each other, the uniformity of the permutation function implies the following useful properties\cite{ying2017convergence, bertsekas2003convex, horvitz1952generalization}:
\begin{align}
	\bsigma^k(i) \neq&\, \bsigma^k(j),\;\; 1\leq i\neq j \leq N \label{prop1}\\
	\Pr[\ \bsigma^k(i)=n\ ] =&\, \frac{1}{N},\;\;\;\hspace{3.4mm}  1\leq n\leq N \label{prop2}
	\\
	\Pr[\bsigma^k(i+1)=n \,|\, \bsigma^k(1\colon i)] =&\,\left\{
	\begin{aligned}
	\frac{1}{N-i},\;\; & n \notin  \bsigma^k(1{:}i)\\
	0\;\;\;,\;\; & n\in   \bsigma^k(1{:} i)
	\end{aligned}
	\right.\label{prop3}
\end{align}
where $\bsigma^k(1{:} i)$ represents the collection of permuted indices for the samples numbered $1$ through $i$.

Several recent works \cite{recht2012toward, gurbuzbalaban2015random, Shamir2016Without} have pursued justifications for the enhanced  behavior of random reshuffling implementations over independent sampling (with replacement).
{\color{black}  The work \cite{gurbuzbalaban2015random} examined the convergence rate of the learning process under diminishing step-sizes, i.e., $\mu(i)=c/i$, where $c$ is some positive constant. It analytically showed that, for strongly convex objective functions, the convergence rate under random reshuffling can be improved from $O(1/i)$ in vanilla SGD\cite{agarwal2009information} to $O(1/i^2)$.
The incremental gradient methods\cite{bertsekas1997new,gurbuzbalaban2015convergence}, which can be viewed as the deterministic version of random reshuffling, shares similar conclusions, i.e., random reshuffling helps accelerate the convergence rate from $O(1/i)$ to $O(1/i^2)$ under decaying step-sizes.
Also, in the work \cite{Shamir2016Without}, it establishes that random reshuffling will not degrade performance relative to the stochastic gradient descent implementation, provided the number of epochs is not too large.
}

In this work, we focus on a different setting than\cite{recht2012toward, gurbuzbalaban2015random, Shamir2016Without} involving random reshuffling under {\em constant} rather than decaying step-sizes.  In this case, convergence
is only guaranteed to a small neighborhood of the optimizer albeit
at a linear rate. The analysis will  establish analytically that random reshuffling
outperforms independent sampling (with replacement) by showing that the mean-square-error of the iterate at the end of each run in the random reshuffling strategy will be in the order of $O(\mu^2)$. This is a significant improvement over the performance of traditional stochastic gradient descent, which is $O(\mu)$~\cite{yuan2016stochastic}. 
Furthermore, we derive an analytical expression for the steady-state mean-square-error performance of the algorithm, {\color{black} which is exact for quadratic risks and provides a good approximation for general risks. This helps clarify in greater detail the differences between sampling with and without replacement}
We also explain the periodic behavior that is observed in random reshuffling implementations. 

{\color{black}
\subsection{Overview of results}
 
\begin{itemize}
	\item Section \ref{sec.xe.g} provides a stability proof, which shows that under constant step-size random reshuffling will converge into a small neighborhood around the minimizer. The radius of the neighborhood improves  from $O(\mu)$ under uniform sampling to $O(\mu^2)$ under random reshuffling. --- Theorem \ref{lemma.start.pont}.
	
	
	\item Next, we examine more closely the value of the scaling constant in the $O(\mu^2)$ factor by introducing a long term model and deriving an expression for its mean-square-deviation (MSD) performance --- Theorem \ref{main.theorem}. The theorem reveals how the number of samples $N$, step-size $\mu$, and Hessian of the loss function impact performance.
	
	\item In Theorem \ref{theorem.iterations} we provide an expression for an upper bound for the MSD performance at all points close to steady-state. The result of the theorem helps explain the periodic behavior that is observed in random reshuffling implementations. 
	
	\item The mismatch between the original reshuffling and the long model is provided in Lemma \ref{lemma.mismatch}.
	
	\item Inspired by quadratic risks, we simplify the MSD expressions in Theorems \ref{main.theorem} and \ref{theorem.iterations} by using the hyperbolic tanh$(\cdot)$ functions  --- equations \eqref{msd.hyperbolic} and \eqref{23gdsdse}.
	
	\item In equations \eqref{eq.80} -- \eqref{msd.infty}, we show that as the sample size increases, the established MSD expression in Theorem \ref{main.theorem} will regress to the same expression as the uniform sampling case. 
	
\end{itemize}
}
\section{ANALYSIS OF THE STOCHASTIC GRADIENT UNDER RANDOM RESHUFFLING} \label{sec.xe.g}
\subsection{Properties of the Gradient Approximation}
We start by examining the properties of the stochastic gradient $\nabla_w Q(\w^k_{i-1};x_{\bsigma^k(i)})$ under random reshuffling.
One main source of difficulty that we shall encounter in  the analysis of performance under random reshuffling is the fact that a single sample of the stochastic gradient $\nabla_w Q(\w^k_{i-1};x_{\bsigma^k(i)})$ is now a {\em biased} estimate of the true gradient and, moreover, it is no longer independent of past selections, $\bsigma^k(1: i-1)$. This is in contrast to implementations where samples are picked independently at every iteration. Indeed, note that conditioned on previously picked data and on the previous iterate, we have:
\vspace{1mm}
\begin{align}
&\hspace{-4mm}\Ex\big[\nabla_w Q_{\bsigma^k(i)}(\w_{i-1}^k)\,|\, \w_{i-1}^k, \bsigma^k(1\,{:}\,i-1) \big]\nn\\
=&\, \frac{1}{N-i+1} \sum_{n\notin\bsigma^k(1\,{:}\,i-1)}\nabla_w Q(\w_{i-1}^k)
\nn\\
\neq&\, \nabla J(\w_0^k)\label{f23r23bg43}
\end{align}
The difference (\ref{f23r23bg43}) is generally nonzero in view of the definition (\ref{prob-emp-into}).
For the first iteration of every epoch however, it can be verified that the following holds:
\begin{align}
\Ex \left[\nabla_w Q_{\bsigma^k(i)}(\w_0^k) \,\Big|\,\w_0^k \right]
\stackrel{(\ref{prop2})}{=}&\; \frac{1}{N} \sum_{n=1}^N  Q(\w_0^k;x_n)
\nn\\
\stackrel{(\ref{prob-emp-into})}{=}&\; \nabla J(\w_0^k) \label{zero.cross_term}
\end{align}
since at the beginning of one epoch, no data has been selected yet.
Perhaps surprisingly, we will be showing that the biased construction of the stochastic gradient estimate not only does not hurt the performance of the algorithm, but instead significantly improves it. In large part, the analysis will revolve around considering the accuracy of the gradient approximation over an entire epoch, rather than focusing on single samples at a time. Recall that by construction in random reshuffling, every sample is picked once and only once over one epoch. This means that the {\em sample average} (rather than the true mean) of the gradient noise process is zero since
\eq{
	\frac{1}{N}\sum_{i=1}^N \grad_w Q(\w; x_{\bsigma^k(i)}) = \nabla J(\w) \label{grad.define}
}
for any $\w$ and any reshuffling order $\bsigma^k$. This property will become key in the analysis.

\subsection{Convergence Analysis}

We can now establish a key convergence and performance property for the random reshuffling algorithm, which provides solid analytical justification for its observed improved performance in practice.

To begin with, we assume that the risk function satisfies the following conditions, which are automatically satisfied by many learning problems of interest, such as mean-square-error or logistic regression analysis and their regularized versions --- see, e.g., \cite{sayed2014adaptation,sayed2014adaptive,bishop2006pattern,hastie2009elements, theo2008}.
\begin{assumption}[\sc Condition on loss function]
	\label{assumption.1}
	It is assumed that $Q(w;x_n)$ is differentiable and has a $\delta_n$-Lipschitz continuous gradient, i.e., for every $n=1,\ldots,N$ and any $w_1, w_2 \in \real^M$: \vspace{-0.3mm}
	\eq{\label{eq-ass-cost-lc-e}
		\|\grad_w Q(w_1;x_n) - \grad_w Q(w_2;x_n) \| \le \delta_n \|w_1-w_2\|
	}
	where $\delta_n > 0$. We also assume $J(w)$ is $\nu$-strongly convex:
	\eq{\label{eq-ass-cost-sc-e}
		\hspace{-1mm}\Big(\grad_w J(w_1)- \grad_w J(w_2)\Big)\tran (w_1 - w_2) &\ge \nu \|w_1-w_2\|^2
	}
	\qd
\end{assumption}
\vspace{-1mm}
\noindent If we introduce $\delta = \max\{\delta_1, \delta_2, \cdots, \delta_N\}$, then each $\grad_w Q (w;x_n)$ and $\nabla_w J(w)$ are also $\delta$-Lipschitz continuous.

The following theorem focuses on the convergence of the \emph{starting} point of each epoch and establishes in (\ref{lemma.stability2}) that it actually approaches a smaller neighborhood of size $O(\mu^2)$ around $w^{\star}$. Afterwards, using this result, we also show that the same $O(\mu^2)-$performance level holds for {\em all} iterates $\w_i^k$ and not just for the starting points of the epochs.

To simplify the notation, we introduce the constant $\cK$, which is the gradient noise variance at optimal point $w^\star$:
\eq{
	\cK \define \frac{1}{N} \sum_{n=1}^N\left\|\grad_w Q(w^{\star};x_n)\right\|^2 \label{gradient.noise}
}
\begin{theorem}[\sc Stability of starting points]\label{lemma.start.pont}
	{\color{black} Under assumption \ref{assumption.1}, the starting point of each run in \eqref{alg.rr}, i.e., $\w^k_0$, satisfies}
	\eq{
	\limsup_{k\to\infty} \Ex \|\w_0^k - w^\star\|^2 \leq\;& \frac{4\mu^2\delta^2N^2}{\nu^2}\cK
	=O(\mu^2)\label{lemma.stability2}
	}
	when the step-size is sufficiently small, namely, for $\color{black}\mu \leq \frac{\nu}{3\delta^2 N}$\footnote{\color{black} The proof in this theorem is based on the worst case scenario, which implies the inequalities hold for any realizations. Therefore, this proof is also applicable to the deterministic cyclic sampling case.}. {\color{black} The convergence to steady-state regime occurs at an exponential rate, dictated by the parameter:
	\eq{
		\alpha \define 1- \mu\nu N/2
	}}
	\vspace{-2mm}
\end{theorem}
\bp
See Appendix \ref{proof.theorem.1}
\ep

\smallskip
\noindent Having established the stability of the first point of every epoch, we can establish the stability of every point.\vspace{-1mm}
\begin{corollary}[\sc Full Stability] \label{lemma.corollary.1}
Under assumption \ref{assumption.1}, it holds that\eq{
	\limsup_{k\to\infty} \Ex \|\w_{i}^k - w^\star\|^2 =\;&O(\mu^2)\label{lemma.stability3}
}
for all $i$ when the step-size is sufficiently small.
\end{corollary}
\bp
See Appendix \ref{proof.corollary.1}
\ep
\vspace{2mm}
{\color{black}
\noindent With the previous established Theorem \ref{lemma.start.pont},	it is also easy to gain the convergence theorem under decaying step-sizes.\vspace{-1mm}
\begin{corollary}[\sc Convergence under decaying step-sizes] \label{lemma.corollary.2}
	Under assumption \ref{assumption.1} and the decaying step-sizes $\mu(i) = {c}/{(i+1)}$ is employed, the iterate $\w^k_i$ converge to the minimizer $w^\star$ exactly as $i\to\infty$ with $O(1/i^2)$ rate.\vspace{-2mm}
\end{corollary}
\bp
See Appendix \ref{proof.corollary.2}
\ep
}

\section{ILLUSTRATING BEHAVIOR AND PERIODICITY}\label{sec.illus}
In this section we illustrate the theoretical findings so far by numerical simulations. We consider the following logistic regression problem: \vspace{-3mm}
\eq{\label{xcn23bh}
\min_w\quad J(w) = \frac{1}{N}\sum_{n=1}^{N} Q(w;h_n,\gamma(n)),
}
where $h_n\in\real^M$ is the feature vector,  $\gamma(n)\in \{\pm1\}$ is the scalar label, and
\eq{\label{xn8}
Q(w;h_n,\gamma_n) \define \rho \|w\|^2 + \ln\left(1+\exp(-\gamma(n) h_n\tran w)\right).
}
The constant $\rho$ is the regularization parameter. In the first simulation, we compare the performance of the standard stochastic gradient descent (SGD) algorithm (\ref{alg:esgd-intr}) with replacement and the random reshuffling (RR) algorithm \eqref{alg.rr}. We set $N=1000$ and $M=10$. Each
$h_{n}$ is generated from the normal distribution $\cN(0; \Lambda_M)$, where $\Lambda_M$ is a diagonal matrix with each diagonal entry generated from the uniform distribution $\cU(1,10)$. To generate $\gamma(n)$, we first generate an auxiliary random vector $w_0\in \mathbb{R}^{M}$ with each entry following $\cN(0,1)$. Next, we generate $\u(n)$ from a uniform distribution $\cU(0,1)$. If $\u(n) \le 1/(1+\exp(-h_{n}\tran w_0))$ then $\gamma(n)$ is set as $+1$; otherwise $\gamma(n)$ is set as $-1$. We select $\rho=0.1$ during all simulations. Figure \ref{fig:sgd_vs_rr_1en2} illustrates {\color{black}the mean-square-deviation (MSD) performance, i.e., $\Ex\|\w_0^k-w^\star\|^2$,} of the SGD and RR algorithms when $\mu = 0.003$. It is observed that the RR algorithm oscillates during the steady-state regime, and that the MSD at the $\w_0^k$ is the best among all iterates $\{\w_i^k\}_{i=1}^{N-1}$ during epoch $k$. Furthermore, it is also observed that RR has better MSD performance than SGD. Similar observations also occur in Fig. \ref{fig:sgd_vs_rr_1en3}, where $\mu=0.0003$. It is worth noting that the gap between SGD and RR is much larger in Fig. \ref{fig:sgd_vs_rr_1en3} than in Fig. \ref{fig:sgd_vs_rr_1en2}.\vspace{-3mm}

\begin{figure}[htb]
	\centering
	\includegraphics[scale=0.35]{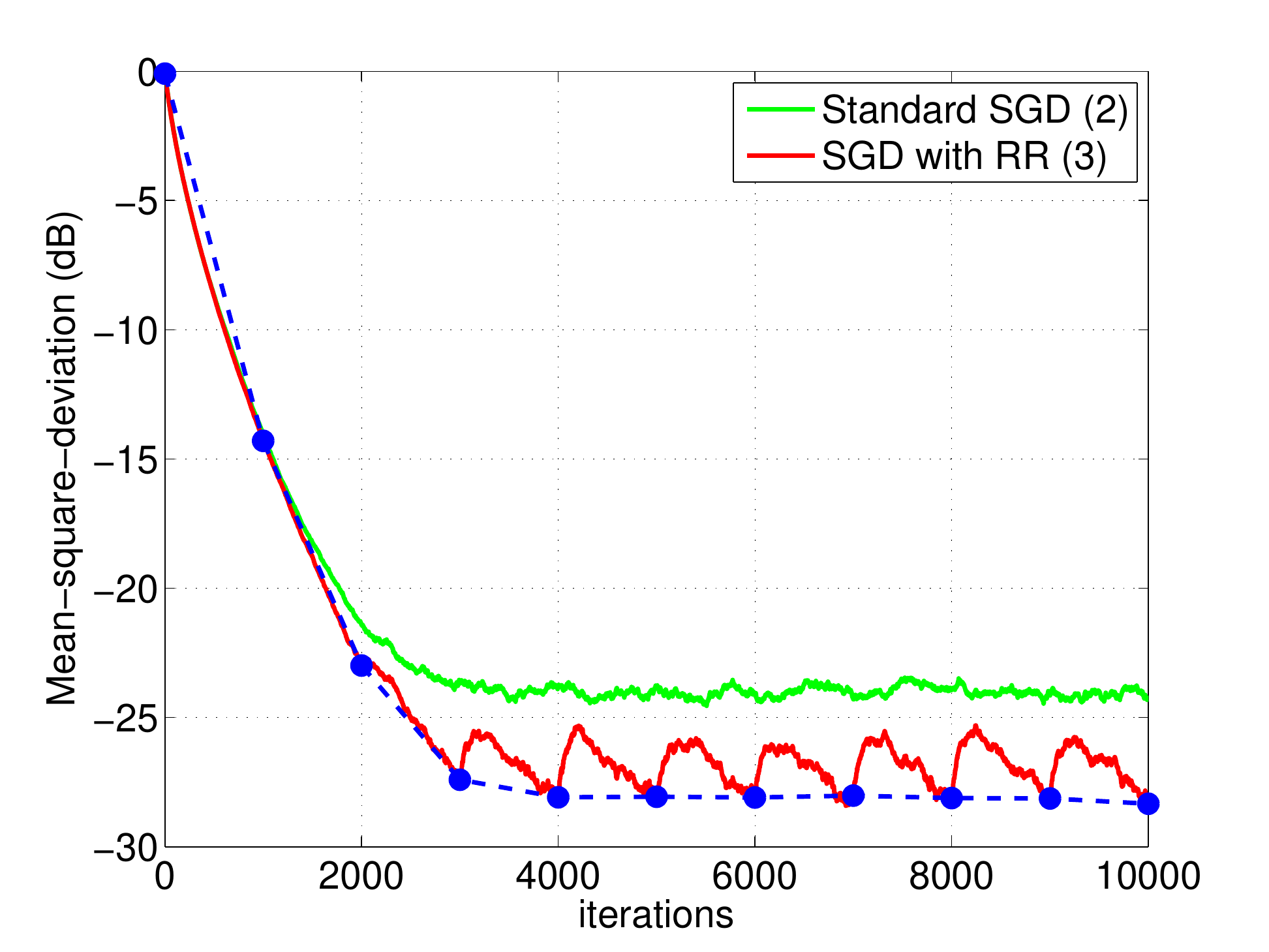}\vspace{-2mm}
	\caption{\small{RR has better mean-square-deviation (MSD) performance, i.e., $\Ex\|\w_0^k-w^\star\|^2$,  than standard SGD when $\mu = 0.003$. The dotted black curve is drawn by connecting the MSD performance at the starting points of the successive epochs.}}\vspace{-2mm}
	\label{fig:sgd_vs_rr_1en2}
\end{figure}

\begin{figure}[htb]
	\centering
	\includegraphics[scale=0.35]{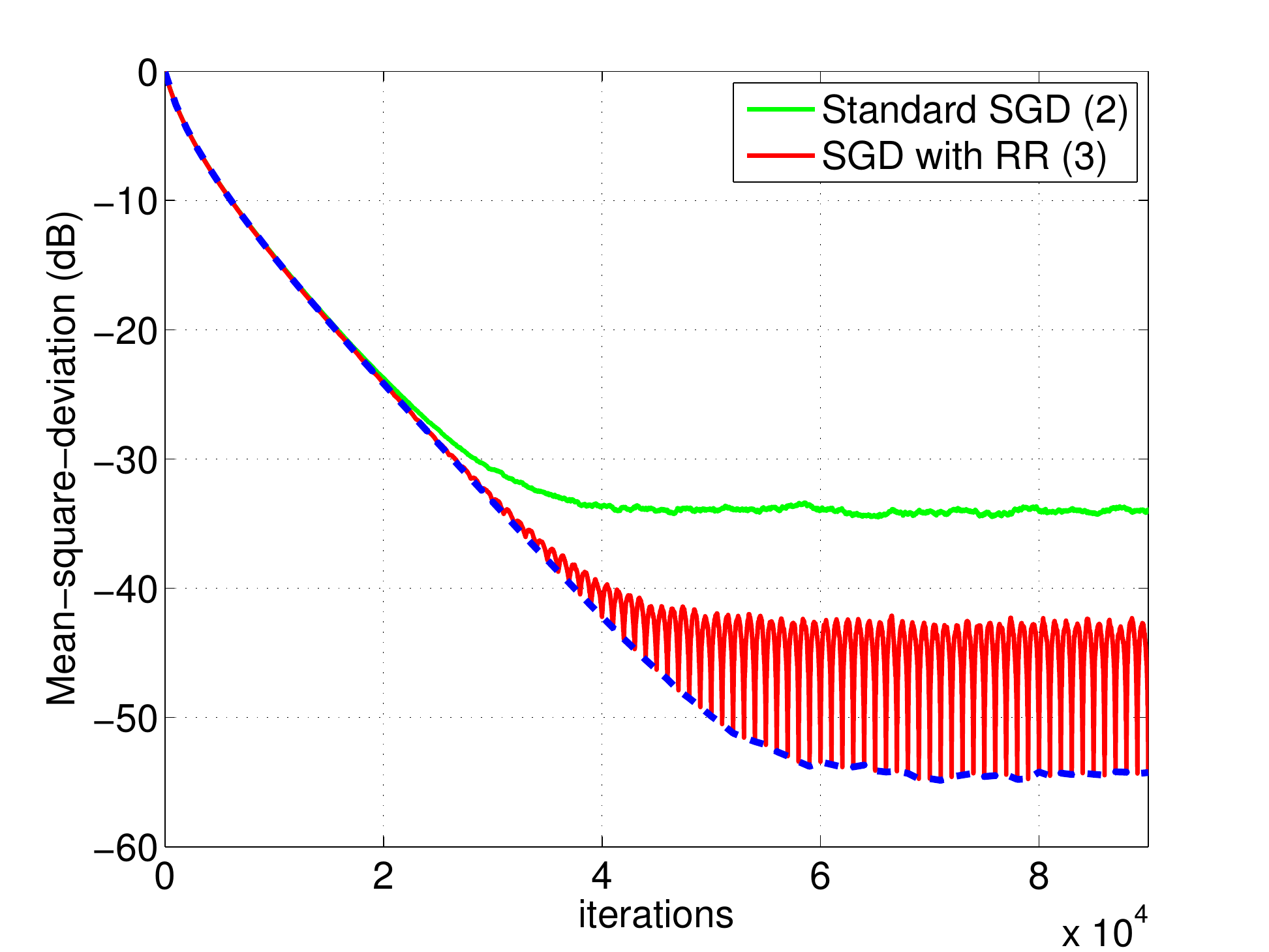}\vspace{-1mm}
	\caption{\small{RR has much better MSD performance, i.e., $\Ex\|\w_0^k-w^\star\|^2$, than standard SGD when $\mu = 0.0003$. The dotted black curve is drawn by connecting the MSD performance at the starting points of the successive epochs.}}\vspace{-2mm}
	\label{fig:sgd_vs_rr_1en3}
\end{figure}

Next, in the second simulation we verify the conclusion that the MSD for the starting point of each epoch for the random reshuffling algorithm, i.e., $\w_0^k$, can achieve $O(\mu^2)$ instead of $O(\mu)$. We still consider the regularized logistic regression problem \eqref{xcn23bh} and \eqref{xn8}, and the same experimental setting.
Recall that in Lemma \ref{lemma.start.pont}, we proved that
\begin{align}
\limsup_{k\to\infty} \Ex\|\widetilde{\w}_{0}^k\|^2
\leq&\,O(\mu^2),
\end{align}
which indicates that when $\mu$ is reduced a factor of 10, the MSD-performance $\Ex\|\widetilde{\w}_{0}^k\|^2$ should be improved by at least $20$ dB. We observe a decay of about 20dB per decade in Fig. \ref{fig:msd_vs_stepsize_20} for a logistic regression problem with $N=25$ data points and 30dB per decade in Fig. \ref{fig:msd_vs_stepsize_30} with $N=1000$.


\begin{figure}[!htb]
	\centering
	\includegraphics[scale=0.36]{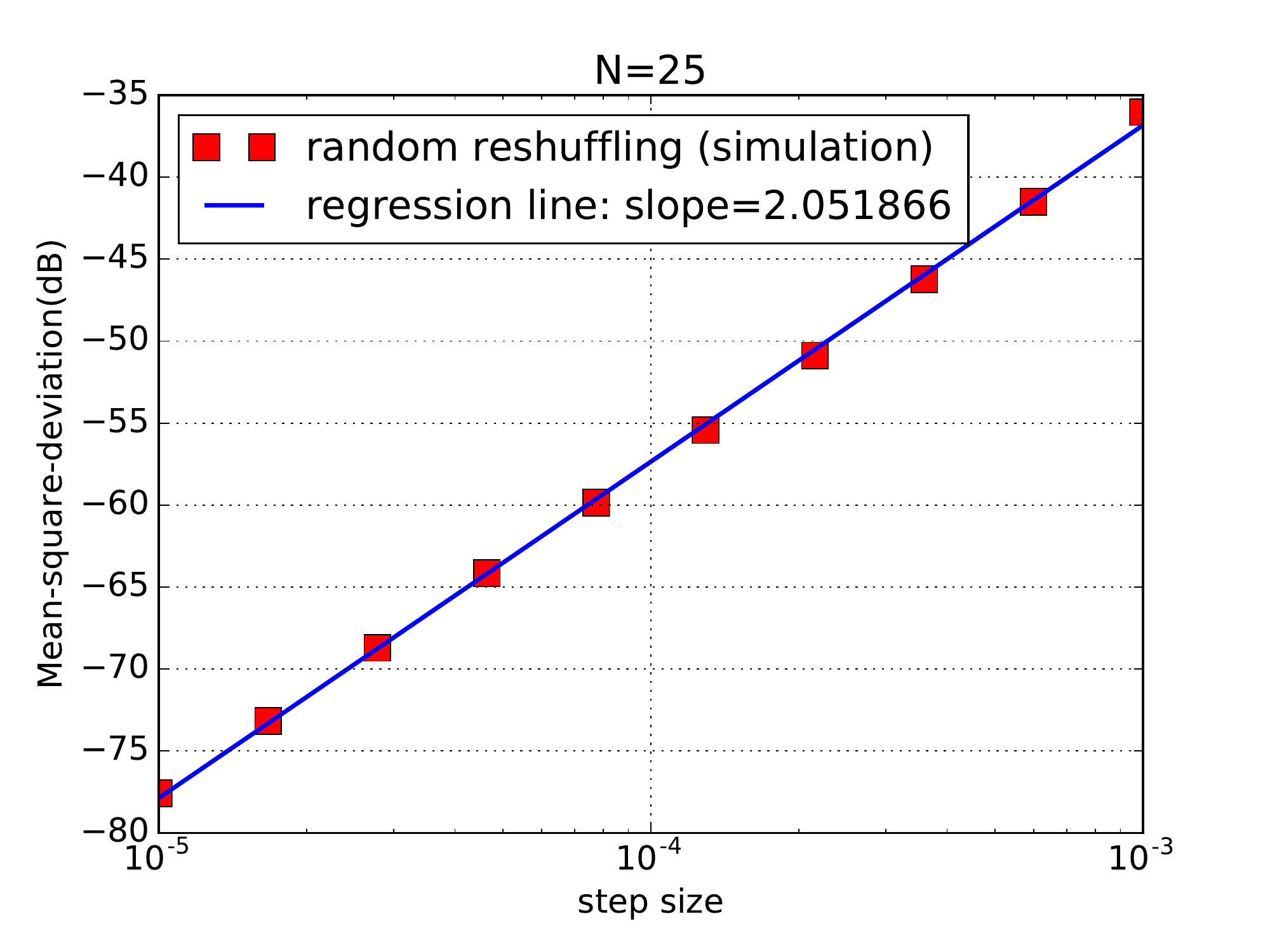}\vspace{-3mm}
	\caption{\small{Mean-square-deviation performance at steady-state versus the step size for a logistic problem involving $N=25$ data points. The slope is around $20$ dB per decade.}}\vspace{-3mm}
	\label{fig:msd_vs_stepsize_20}
\end{figure}
\begin{figure}[!htb]
	\centering
	\includegraphics[scale=0.36]{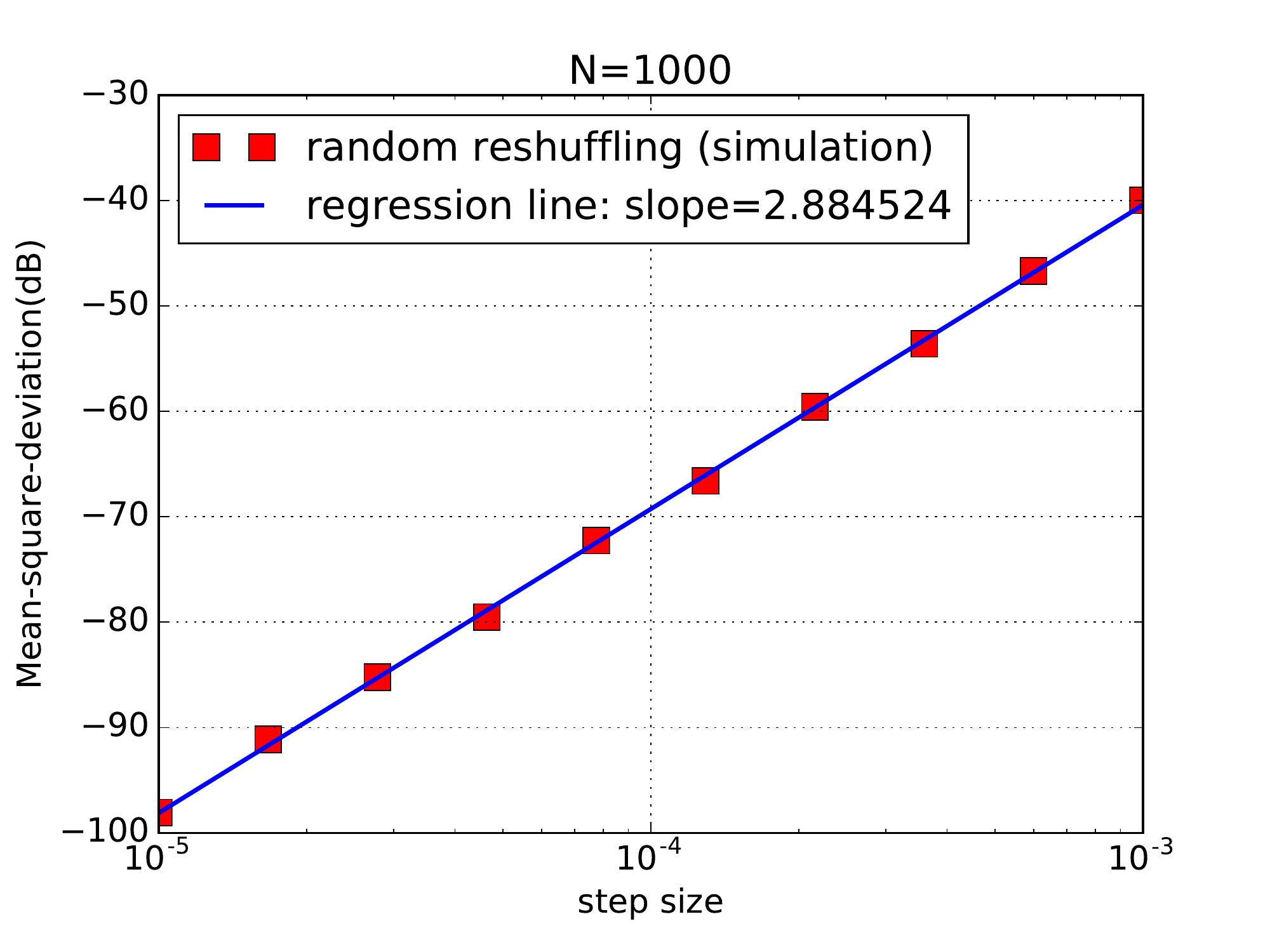}\vspace{-3mm}
	\caption{\small{Mean-square-deviation performance at steady-state versus the step size for a logistic problem involving $N=1000$ data points. The slope is around $30$ dB per decade.}}
	\label{fig:msd_vs_stepsize_30}
\end{figure}


\smallskip
\section{INTRODUCING A  LONG-TERM MODEL}
We proved in the earlier sections  that the mean-square error under random reshuffling approaches a small $O(\mu^2)-$neighborhood around the minimizer. Our objective now is  to assess more accurately  the size of the constant that multiplies $\mu^2$ in the $O(\mu^2)$ result, and examine how this constant may depend on various parameters including the amount of data, $N$, and the form of the loss function $Q$.
To do that, we proceed in two steps. First, we introduce  an auxiliary long-term model in (\ref{eq.long_term.rec}) below and subsequently determine how far the performance of this model is from the original system described by \eqref{eq.error_expand} further ahead.\vspace{-3mm}
\subsection{Error Dynamics}\label{subsec.error}
In order to quantify the performance of the random reshuffling implementation more accurately than the $O(\mu^2)-$figure obtained earlier, we will need to impose a condition on the smoothness of the Hessian matrix of the risk function.
\begin{assumption}[\sc Hessian is Lipschitz continuous]
	\label{assumption.2}
	The risk function $J(w)$ has a Lipschitz continuous Hessian matrix, i.e., there exists a constant $\kappa\geq0$, such that
	\be
	\|\nabla^2_w J(w_1) - \nabla^2_w J(w_2) \| \leq \kappa \|w_1-w_2\|
	\label{hessian.lipschitz}
	\ee\qd
\end{assumption}
\vspace{-1mm}
\noindent Under this assumption, the gradient vector, $\nabla_w J(w)$, can be expressed in Taylor expansion in the form\cite[p. 378]{ sayed2014adaptation}:
\be
\nabla_w J(w) =  \nabla^2_w J(w^\star) (w-w^\star) + \xi(w),\;\;\;\; \forall w
\ee
where the residual term satisfies:
\be
\|\xi(w)\| \leq \frac{\kappa}{2}\|w-w^\star\|^2 \label{3h98u.ni}
\ee

\noindent 
As such, we can rewrite algorithm (\ref{alg.rr}) in the form: \begin{align}
\tw_i^k =& \tw_{i-1}^k + \mu \nabla_w J(\w_{i-1}^k) \nn\\
&\;\;+ \mu\Big(\grad_w Q(\w_{i-1}^k;x_{\bsigma^k(i)})- \nabla_w J(\w_{i-1}^k)\Big)\nn\\
=& \tw_{i-1}^k - \mu  \nabla^2_w J(\w^\star) \tw_{i-1}^k + \mu\xi(\w_{i-1}^k)\nn\\
&\;\;+\mu\Big(\grad_w Q(\w_{i-1}^k;x_{\bsigma^k(i)})- \nabla_w J(\w_{i-1}^k)\Big)\label{231398u.123932}
 \end{align}
To ease the notation, we introduce the Hessian matrix $H$ and the gradient noise process:
\eq{
	H \define& \nabla^2_w J(w^\star) \nn\\
	 s_{\bsigma^k(i)}(\w_{i-1}^k)\define&\grad_w Q(\w_{i-1}^k;x_{\bsigma^k(i)})- \nabla_w J(\w_{i-1}^k)
}
so that \eqref{231398u.123932} is simplified as:
\eq{
\tw_i^k =(I-\mu H) \tw_{i-1}^k + \mu\xi(\w_{i-1}^k) + \mu s_{\bsigma^k(i)}(\w_{i-1}^k)
\label{231398u.1239}
}
Now property \eqref{zero.cross_term} motivates us to expand (\ref{231398u.1239}) into the following error recursion by adding and subtracting the same gradient noise term evaluated at $\w_0^k$:
\begin{align}
\widetilde{\w}_i^k =& (I-\mu H)\widetilde{\w}_{i-1}^k + \mu
s_{\bsigma^k(i)}(\w_{0}^k)\nn\\          \;\;&\;\;\;{}+\underbrace{\mu \big(
	s_{\bsigma^k(i)}(\w_{i-1}^k)- s_{\bsigma^k(i)}(\w_{0}^k) \big)}_{\rm noise\ mismatch} + \mu\xi(\w_{i-1}^k)
\label{eq.origin.onestep}
\end{align}
Iterating (\ref{eq.origin.onestep}) and using (\ref{bound.condtion}) we can establish the following useful relation, which we  call upon in the sequel:
\begin{align}
\widetilde{\w}_0^{k+1}
=&\;\; (I-\mu H)^N \widetilde{\w}_{0}^k + \mu \sum^N_{i=1} ( I - \mu H )^{N-i} s_{\bsigma^k(i)}(\w_0^k) \nn\\
&\;\; + \mu \sum^N_{i=1} ( I - \mu H )^{N-i} \left(s_{\bsigma^k(i)}(\w_{i-1}^k)- s_{\bsigma^k(i)}(\w_{0}^k)\right) \nn\\
&\;\; +  \mu \sum^N_{i=1} ( I - \mu H )^{N-i} \xi(\w_{i-1}^k)
\label{eq.error_expand}
\end{align}
Note that recursion (\ref{eq.error_expand}) relates $\widetilde{\w}_0^k$ to $\widetilde{\w}_0^{k+1}$, which are the starting points of two successive epochs. In this way, we have now transformed recursion (\ref{alg.rr}), which runs from one sample to another within the same epoch, into  a relation that runs from one starting point to another over two successive epochs.

To proceed, we will ignore the last two terms in (\ref{eq.error_expand}) and consider the following approximate model, which we shall refer to as a {\em long-term} model.
\be
\widetilde{\w}_0^{\prime k+1}
= (I-\mu H)^N\widetilde{\w}_{0}^{\prime k} \color{black}+\color{black} \mu\underbrace{\sum_{i=1}^{N}(I-\mu H)^{N-i} s_{\bsigma^k(i)}(\w_{0}^k)}_{\define s'(\w_0^k)} \label{eq.long_term.rec}
\ee
Obviously, the state evolution will be different than (\ref{eq.error_expand}) and is therefore denoted by the prime notation, $\widetilde{\w}_0^{\prime k}$. Observe, however, that in model (\ref{eq.long_term.rec}) the gradient noise process is still being evaluated at the original state vector, $\w_0^k$, and not at the new state vector, $\w_0^{\prime k}$.

\subsection{Performance of the Long-Term Model across Epochs}\label{subsec.perform.epoch}
Note that the gradient noise $\s'(\w_0^k)$ in (\ref{eq.long_term.rec}) has the form of a weighted sum over one epoch. This noise clearly satisfies the property:
%
%
\begin{align}
	\Ex [\,s'(\w_0^k) \,|\, \w_0^k\,]=&\; 0 \label{noise.zero}
\end{align}
We also know that  $s'(\w_0^k)$ satisfies the Markov property, i.e., it is independent of all previous $\w^{k'}_i$ and $\bsigma^{k'}(\cdot)$, where $k'<k$,  conditioned on  $\w_0^k$.
To motivate the next lemma consider the following auxiliary setting.

Assume we have a collection of $N$ vectors $\{x_i\}$ in $\real^2$ whose sum is zero. We define a random walk over these vectors in the following manner. At each time instant, we select a random vector $x_{\n_i}$ uniformly and with replacement from this set and move from the current location along the vector $x_{\n_i}$ to the next  location. If we keep repeating this construction, we obtain behavior that is represented by the right plot in Fig. 5. Assume instead that we repeat the same experiment except that now we assume the data $\{x_i\}$ is first reshuffled and then vectors $x_{\bsigma(i)}$ are selected uniformly without replacement. Because of the zero sum property, and because sampling is now performed without replacement, we find that in this second implementation we always return to the origin after $N$ selections. This situation is illustrated in the left plot of the same Fig. \ref{fig:random_walk}. The next lemma considers this scenario and provides useful expressions that allow us to estimate the expected location after $1, 2$ or more (unitl $N-1$) movements. These results will be used in the sequel in our analysis of the performance of stochastic learning under RR.

\begin{figure}[htb]
	\centering
	\includegraphics[scale=0.35]{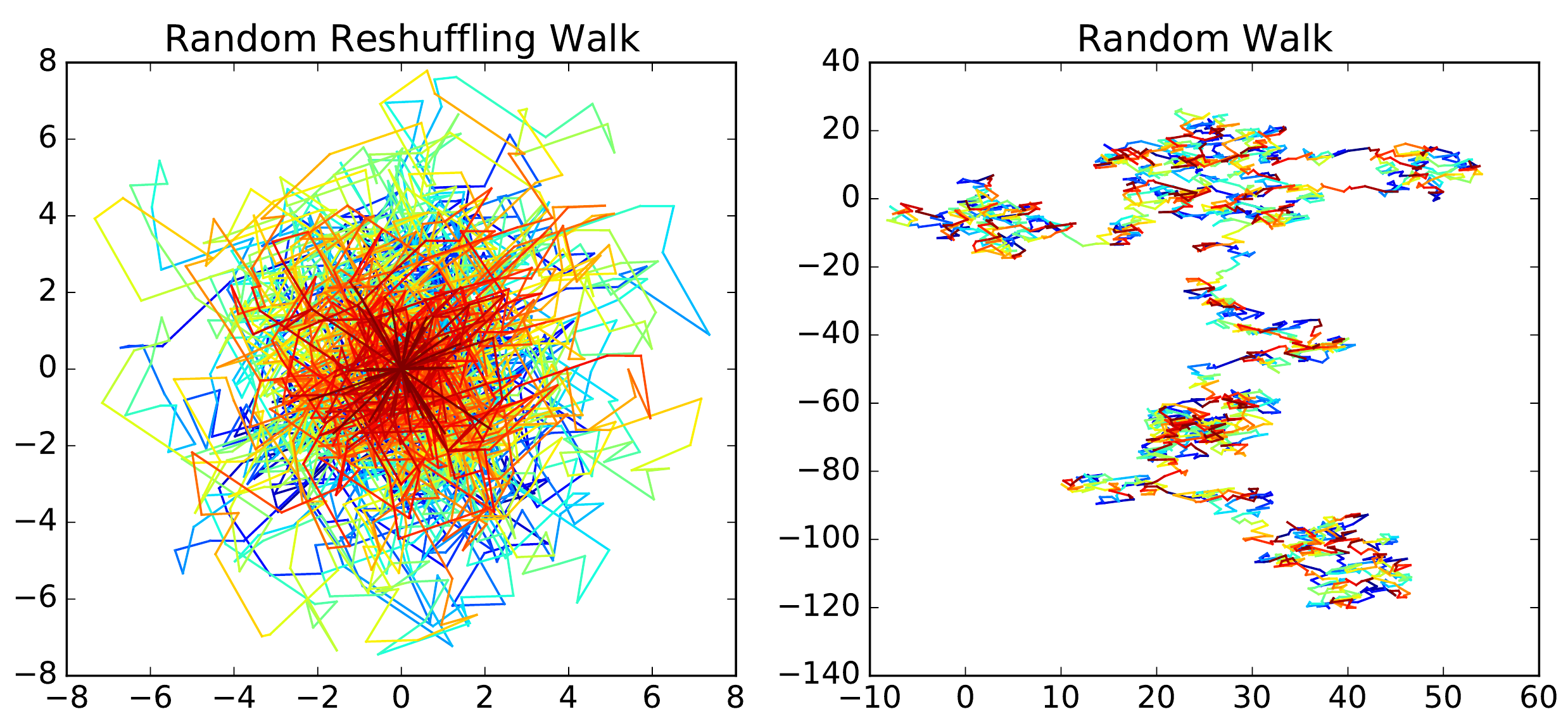}
	\caption{\small{Random walk versus Random reshuffling walk. The lines with same color represent all $i$-th choices walk in different epochs. }}
	\label{fig:random_walk}
\end{figure}

\begin{lemma} \label{lemma.2}
Suppose we have a set of $N$ vectors $X = \{x_i\}^N_{i=1}$ with the constraint
\(
\sum_{i=1}^{N} x_i =  0
\). Assume the elements of $X$ are randomly reshuffled and then selected uniformly without replacement. Let $\beta$ be any nonnegative constant, $B$ be any symmetric positive semi-definite matrix, and introduce  \eq{
	R_x\define&\frac{1}{N}\sum_{i=1}^N x_ix_i\tran\\
	{\rm Var}(X) \define& \frac{1}{N} \sum_{i=1}^N\|x_i\|^2 = \Tr(R_x)
}
 Define the  following functions for any $1\leq n\leq N$:
\eq{
f(n;X,\!\beta) \define& \Ex \left\|\sum_{j=1}^{n}\beta^{n-j} x_{\bsigma(j)}\right\|^2\\
F(n;X,B) \define & \Ex \left[\sum_{j=1}^nB^{n-j} x_{\bsigma(j)}\right]\left[\sum_{j=1}^n x_{\bsigma(j)}\tran B^{n-j}\right]
}
It then holds that \eq{
f(n;X,\!\beta)=&\, \frac{(\sum_{i=0}^{n-1}\beta^{2i})N-(\sum_{i=0}^{n-1}\beta^i)^2}{N-1}{\rm Var}(X) \label{eq.lemma2}\\
	F(n;X,B)=& \frac{\left[\displaystyle\sum_{i=0}^{n-1}B^i R_xB^i\right]\!\!N \!-\! \left[\displaystyle\sum_{i=0}^{n-1}B^i\right] \!R_x \!\left[\displaystyle\sum_{i=0}^{n-1}B^i\right] }{N-1}
	\label{eq.lemma2.2}
}
\end{lemma}
\begin{proof}
	The proof is provided in Appendix \ref{app.noise}.
\end{proof}\vspace{2mm}

We now return to the stochastic gradient implementation under random reshuffling. Recall from \eqref{grad.define} that the stochastic gradient satisfies the zero sample mean property so that
\eq{
\sum_{i=1}^N s_{\bsigma^k(i)}(\w) = 0
}
at any given point $\w$. Applying Lemma \ref{lemma.2}, we readily conclude that
\begin{align}
&\hspace{-5mm}\Ex [s'(\w_0^k)s'(\w_0^k)\tran\,|\,\w_0^k]\nn\\
=&\; \frac{N\left(\sum_{i=0}^{N-1} (I-\mu H)^iR_s^k(I-\mu H)^i\right)}{N-1}\nn\\
&\;\; {}- \frac{\big[\sum_{i=0}^{N-1}(I-\mu H)^i\big] R_s^k \big[\sum_{i=0}^{N-1}(I-\mu H)^i\big]}{N-1} \label{agg.noise.result}
\end{align}
where
\be
	R_s^k \define\frac{1}{N} \sum_{n=1}^N \s_n(\w_0^k) \s_n(\w_0^k)\tran
\ee
Similarly, we conclude for the gradient noise at the optimal $w^\star$:
\eq{
	R_s^{\prime \star}\define&\Ex [s'(\w^\star)s'(\w^\star)\tran]\nn\\
	=&\; \frac{N\left(\sum_{i=0}^{N-1} (I-\mu H)^iR^\star_s(I-\mu H)^i\right)}{N-1}\nn\\
	&\; {}- \frac{\big[\sum_{i=0}^{N-1}(I-\mu H)^i\big] R^{ \star}_s \big[\sum_{i=0}^{N-1}(I-\mu H)^i\big]}{N-1} \label{eq.r_s.star}
}
where \eq{	R^\star_s = &\; \frac{1}{N}\sum_{i=0}^N \nabla Q(w^{\star}; x_i)\nabla Q(w^{\star}; x_i)\tran}
\begin{theorem}[\sc Performance of Long-term Model across Epochs] \label{main.theorem}
	Under assumptions \ref{assumption.1} and \ref{assumption.2}, when the step size $\mu$ is sufficiently small, {\color{black} namely, for $\mu\leq{1}/{\delta}$,} the mean-square-deviation (MSD) of the long term model (\ref{eq.long_term.rec}) is given by
	\begin{align}
		{\rm MSD}_{\rm RR}^{\rm lt}\define& \limsup_{k\to\infty}\| \w^{\prime k}_0 - w^\star\|^2\nn\\
		=\,& \mu^2 \Tr\left((I- (I-\mu H)^{2N})^{-1}R^{\prime \star}_s\right)+ O(\mu^4)\label{eq.msd.reshuffle}
	\end{align}
	{\color{black} The convergence to steady-state regime occurs at an exponential rate, dictated by the parameter:
		\eq{
			\alpha \define  (1 - \mu \lambda_{\min}(H))^{2N} \approx 1 - 2\mu \lambda_{\min}(H) N
	}}\vspace{-5mm}
\end{theorem}
\begin{proof}See Appendix \ref{app.main.theorem}. 
\end{proof}


The simulations in Fig. \ref{fig:rr_lr} show that the MSD expression (\ref{eq.msd.reshuffle}) fits well the performance of the original random reshuffling algorithm. We will establish this fact analytically in the sequel. For now, the simulation is simply confirming that the performance of the long-term model is a good indication of the performance of the original stochastic gradient implementation under RR. \vspace{-3mm}
\begin{figure}[htb]
	\centering
	\includegraphics[scale=0.37]{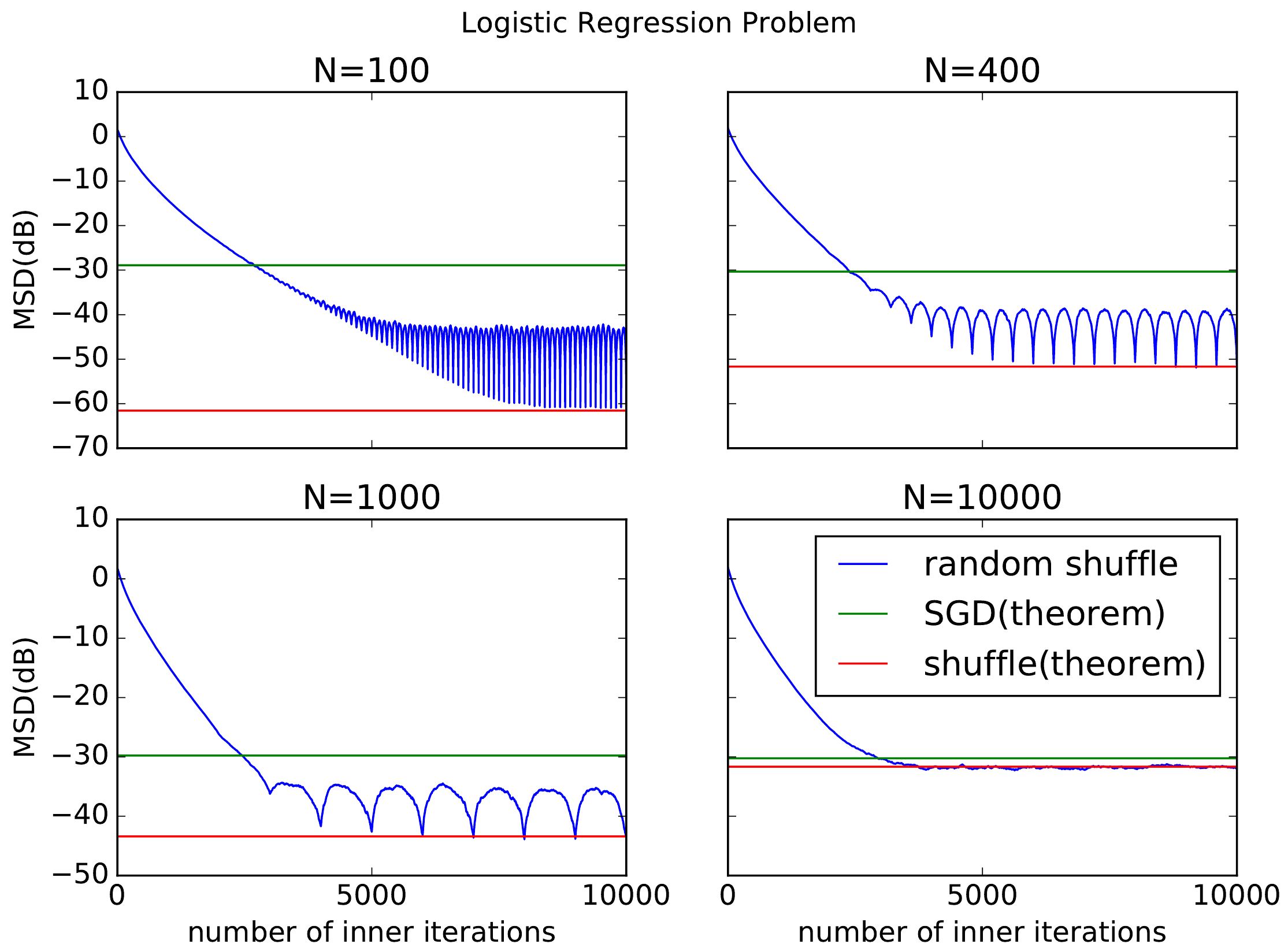}\vspace{-3mm}
	\caption{\small{Mean-square-deviation perfromance of random reshuffling algorithm curve on least-mean-square cost function}}\vspace{-3mm}
	\label{fig:rr_lr}
\end{figure}
\subsection{Performance of the Long-Term Model over Iterations} \label{subsec.perform.time}
In the previous section we examined the performance of the long-term model at the starting points of successive epochs. In this section, we examine the performance of the same model at any iterate $\w_i^k$ as time approaches $\infty$. This analysis will help explain the oscillations that are observed in the learning curves in the simulations. {\color{black} First, similar to \eqref{eq.r_s.star}, we need to determine the covariance matrix $R^{\prime\star}_{s,i}$ for any $i$. From Lemma \ref{lemma.2}, we immediately get that }
\eq{
	R^{\prime\star}_{s,i}\define&\Ex s_i'(w^\star) s_i'(w^\star)\tran \nn\\
	=&\; \frac{N\left(\sum_{j=0}^{i-1} (I-\mu H)^jR_s^\star(I-\mu H)^j\right)}{N-1}\nn\\
	&\;\; {}- \frac{\big[\sum_{j=0}^{i-1}(I-\mu H)^j\big] R_s^\star \big[\sum_{j=0}^{i-1}(I-\mu H)^j\big]}{N-1} 
}
{\color{black}
\begin{theorem}[\sc Performance\! Upper-bound\! for\! Long- Term Model]\label{theorem.iterations}
	Under assumptions \ref{assumption.1} and \ref{assumption.2}, when the step size $\mu$ satisfies $\mu\leq\frac{2}{\delta+\nu}$, the upper-bound of mean-square-deviation (MSD) of the long term model (\ref{eq.long_term.rec}) at all iterations is given by
	\eq{
		&\hspace{-8mm}\lim_{k\to\infty}\Ex\|\widetilde{\w}_{i}^{\prime k}\|^2\nn\\
		\leq&(1-\mu\nu)^{2i}\mu^2\Tr\left(\big(I-(I-\mu H)^{2N}\big)^{-1}  R_{s}^{\prime \star}\right)\nn\\
		&	 +\Big(1\!-\!(1\!-\!\mu\nu)^{2i}\Big)\mu^2\Tr\left(\big(I\!-\!(I\!-\!\mu\nu)^{2i}\big)^{-1}  R_{s,i}^{\prime \star}\right) \label{2389g.cs}\\
		\define& \eta_i {\rm MSD}_{\rm RR}^{\rm lt} + (1-\eta_i) {\rm MSD}_{\rm RR,i}^{\rm lt} \label{23x.gewsd}
	}
\end{theorem}	
\bp See Appendix \ref{app.theorem.iterations}.\ep
}

We need to point out unlike that \eqref{eq.msd.reshuffle}, expression \eqref{89n.sdg} is an upper-bound rather than an actual performance expression. Still, this bound can help provide useful insights on the periodic behavior that is observed in the simulations. The expression \eqref{2389g.cs} on the right-hand side is a convex combination of two performance measures as defined in \eqref{23x.gewsd}, where the second term is always larger than the first term but approaching it as $i$ increases towards $N$. This behavior will  become clearer later in the context of an example and the hyperbolic representation in section \ref{subsec.hyper}.

{\color{black} 
	Before we continue, we would like to comment on the convergence curve under random reshuffling. Unlike the convergence curve under uniform sampling, we observe periodic fluctuations under random reshuffling in Figures \ref{fig:sgd_vs_rr_1en3} and \ref{fig:rr_lr}. The main reason for this behavior is the fact that the gradient noise is no longer i.i.d. in steady-state. Specifically, the noise variance is now a function of the iterate and it assumes its lowest value at the beginning and end of every epoch. In lemma \ref{lemma.2}, we show that the variance of the random walk process resulting from random reshuffling at each iteration $n$ in Eq. \eqref{eq.lemma2}. We plot the function for $N=20$ and ${\rm Var}(x)=1$ in Fig.~\ref{fig:var_func}. Since the mean-square performance of the algorithm is related to the variance of the gradient noise, it is expected that this bell-shape behavior will be reflected in to the MSD curve as well, thus, resulting in better performance at the beginning and end of every epoch.

	
	\vspace{-3mm}
\begin{figure}[htb]
	\centering
	\includegraphics[scale=0.43]{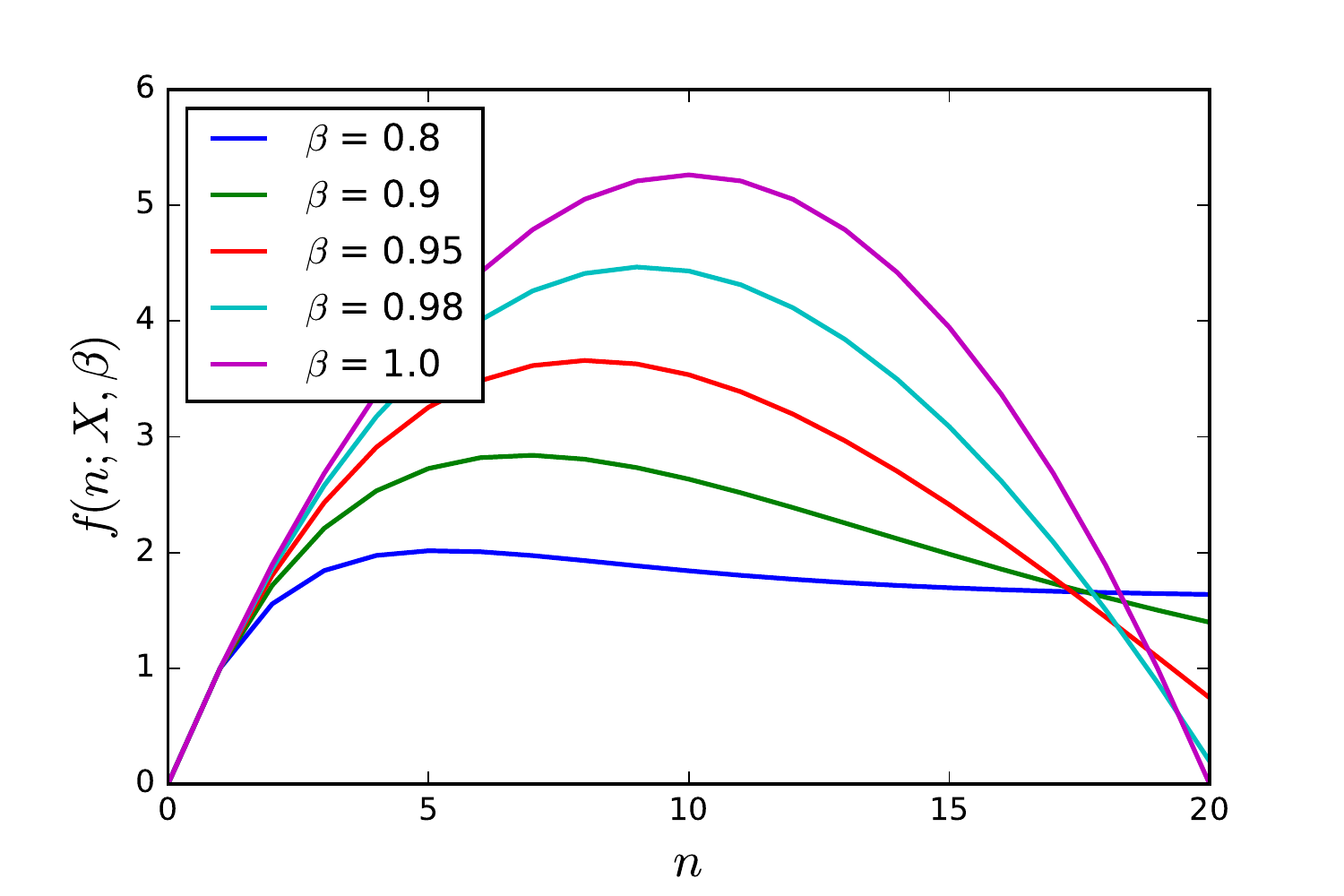}\vspace{-2mm}
	\caption{\small{The variance function $f(n;X,\beta)$ at \eqref{eq.lemma2} versus $n$ with different $\beta$ value.}}\vspace{-3mm}
	\label{fig:var_func}\vspace{-3mm}
\end{figure}
} 

{\color{black}
\subsection{Mismatch Bound}
Now we provide an upper bound on the mismatch between the long-term model (\ref{eq.long_term.rec}) and the original algorithm (\ref{alg.rr}).
\begin{lemma}[\sc Mismatch Bound] \label{lemma.mismatch}
	After long enough iterations, i.e., $k\gg 1$, the difference between the long term model trajectory (\ref{eq.long_term.rec}) and the original trajectory (\ref{alg.rr}) is
	\begin{align}
	\limsup_{k\to\infty} &\;\Ex\|\widetilde{\w}_0^{\prime k} - \widetilde{\w}_0^{k}\|^2 \leq \frac{4\mu^2\delta^2N^2}{\nu^2(N-1)}\cK +O(\mu^3)
	\label{eq.mismatch}
	\end{align}
\end{lemma}
{\bf Proof:}
See Appendix  \ref{mismatch.proof}.\hfill \qd
}

\smallskip
\section{QUADRATIC RISKS AND HYPERBOLIC REPRESENTATION}
{\color{black} Lastly,} we consider an example involving a quadratic (least-squares) risk  to show  that, in this case,  the long-term model provides the exact MSD for the original algorithm. The analysis will also provide some insights into expression (\ref{eq.msd.reshuffle}).  {\color{black} It also motivates a hyperbolic representation for the MSD, which helps provides some more insights into the MSD behavior.}

\subsection{Quadratic Risks}
Thus, consider the following quadratic risk function:\vspace{-1mm}
\be
	\min_{w}  J(w) = \frac{1}{2N} \sum_{n=1}^N\|Aw-x_n\|^2
\ee
where $A$ has full column rank. We have:
\eq{
	\nabla_w J(w) = &\, A\tran A w - A\tran\underbrace{\left(\frac{1}{N}\sum_{n=1}^N  x_n\right)}_{\define \bar{x}}\\[-0.5mm]
	\nabla_w Q(w; x_n) =&\,A\tran A w -A\tran x_n\\
	\nabla^2 J(w_{i}^k) =&\, A\tran A\\
	s_n(w) 
	=&\, A\tran (x_n - \bar{x})
}
Since the gradient noise $s_n(\w)$ is independent of $\w$, we have
\eq{
	s_n(\w_i^k)-s_n(\w_0^k) \equiv  0
}
Moreover, since the risk is quadratic, it also holds that  \be\xi(w)\equiv 0\ee
Therefore, the long-term model is exactly the same as the original algorithm. For this example, we can calculate the following quantities:
\eq{
	w^\star =&\, (A\tran A)^{-1} A\tran \bar{x}\\
	R_s^\star=&\,A\tran \frac{1}{N}\sum_{n=1}^N(x_n - \bar{x})(x_n - \bar{x})\tran A=A\tran R_{xx} A\\
	{\rm Var}(x) = & \frac{1}{N}\sum_{n=1}^N \|x_n -\bar{x}\|^2\,\\
	I-\mu H =&\, I - \mu A\tran A
}
In special case when the columns of $A$ are orthogonal and normalized, i.e., $A\tran A = I$, we can  simplify the MSD expression \eqref{eq.msd.reshuffle} by noting that 
\eq{
	&\hspace{-4mm}R_s^{\prime \star}\nn\\[-2mm]
	=&\, \frac{1}{N-1} \!\left(\!N\sum_{i=0}^{N-1} (1-\mu)^{2i} - \Big(\sum_{i=0}^{N-1} (1-\mu)^{i}\Big)^2\right)\!A\tran R_{xx}A
	\nn\\
	=&\,\frac{1}{N-1}\!\!\left(\!\frac{N(1-(1\!-\!\mu)^{2N})}{2\mu-\mu^2} - \frac{\left(1 -(1\!-\!\mu)^{N}\right)^2}{\mu^2}\right)\!A\tran R_{xx}A
}
and, hence,
\eq{
{\rm MSD}_{\rm RR} =&\,\mu^2\Tr\left((1 - (1-\mu )^{2N})^{-1} R^{\prime \star}_s\right)\nn\\
=&\, \frac{\mu^2}{N-1}\left(\frac{N}{2\mu-\mu^2}- \frac{(1 -(1-\mu)^{N})^2}{\mu^2(1 -(1-\mu)^{2N})}\right){\rm Var}(x)\nn\\
=&\,\frac{\mu^2}{N-1}\left(\frac{N}{2\mu-\mu^2}- \frac{1 -(1-\mu)^{N}}{\mu^2(1 +(1-\mu)^{N})}\right){\rm Var}(x)}
In order to provide further insights on this MSD expression, 
we simplify it under a  small $\mu$ assumption. We could introduce the Taylor series:  
\eq{
	(1-\mu)^N =&\, 1  - N\mu  + O(N^2\mu^2)\label{gh23} 
}
However, this approximation can be bad if $N$ is large, which is not uncommon in big data. Instead, we appeal to:
\eq{
	(1-\mu)^{N} = e^{N\ln(1-\mu)} = e^{-\mu N + O(\mu^2 N)} \approx e^{-\mu N} \label{exp.approx}
}
Notice it is $O(\mu^2 N)$ instead of $O(\mu^2 N^2)$, and therefore  \eqref{exp.approx} is a tighter approximation than \eqref{gh23} when $N$ is large.
Based on this, we further approximate:
\eq{
	\frac{1 -(1-\mu)^{N}}{1 +(1-\mu)^{N}} \approx {\tanh}(\mu N /2)
}
and arrive at the simplified expression:
\eq{
	{\rm MSD}_{\rm RR} \approx& \frac{\mu}{N-1}\left(\frac{N}{2} - \frac{\tanh(\frac{\mu N}{2})}{\mu}\right) {\rm Var}(x)\nn\\
	=&\frac{\mu}{2}\frac{N}{N-1} \left(1- \frac{2}{\mu N}\tanh\Big(\frac{\mu N}{2}\Big)\right){\rm Var}(x) \label{simply.msd}
}
For comparison purposes, we know that a simplified expression for MSD under uniform sampling has the following expression\cite{yuan2016stochastic}:
\eq{
	{\rm MSD}_{\rm us} =&\frac{\mu}{2}{\rm Var}(x)
}
Hence, the random reshuffling case has an extra multiplicative factor:
\eq{
	m_{\rm RR} \define\frac{N}{N-1} \left(1- \frac{2}{\mu N}\tanh\Big(\frac{\mu N}{2}\Big)\right)	
}
We plot $m_{\rm RR}$ versus $\mu N$ in the left plot of Fig.~\ref{fig:rr_factor} where we ignore $\frac{N}{N-1}$.
Now it is clear from the figure that the smaller the step size $\mu$ or the smaller sample size $N$ are, the larger the improvement in performance is. In contrast,  when $\mu N$ goes to infinity, the term $m_{\rm RR}$ will converge to $1$, i.e., the same performance as uniform sampling situation, which is consistent with the infinite-horizon case.
Lastly, noting that 
\eq{
	&\hspace{-4mm}R_{s, i}^{\prime \star}\nn\\
	=&\, \frac{1}{N-1}\!\! \left(\!N\sum_{j=0}^{i-1} (1-\mu)^{2j} - \Big(\sum_{j=0}^{i-1} (1-\mu)^{j}\Big)^2\right)\!A\tran R_{xx}A
	\nn\\
	=&\,\frac{1}{N-1}\!\!\left(\!\frac{N(1-(1-\mu)^{2i})}{2\mu-\mu^2} - \frac{\left(1 -(1-\mu)^{i}\right)^2}{\mu^2}\right)\!A\tran R_{xx}A\label{g23g.xc}
}
and using the approximation \eqref{exp.approx}:
\eq{
	\hspace{-3mm}R_{s, i}^{\prime \star}\approx\frac{N}{N-1} \left(\frac{1- e^{-2\mu i}}{2\mu} - \frac{(1- e^{-\mu i})^2}{\mu^2 N}\right)A\tran R_{xx}A
}
Substituting in \eqref{2389g.cs}, we get for $i\in [1,N]$:
\eq{
&\hspace{-5mm}\lim_{k\to\infty}\Ex\|\widetilde{\w}_{i}^{\prime k}\|^2\nn\\
\approx&
e^{-2\mu i}\frac{\mu}{2}\frac{N}{N-1} \left(1- \frac{2}{\mu N}\tanh\left(\frac{\mu N}{2}\right)\right){\rm Var}(x)\nn\\
&\; + (1-e^{-2\mu i})\frac{\mu}{2}\underbrace{\frac{N}{N-1} \left(1- \frac{2}{\mu N}\tanh\left(\frac{\mu i}{2}\right)\right)}_{\define m_{\rm RR}(i)}{\rm Var}(x) \raisetag{4mm}\label{x2.dsy3g}
}
Since $\tanh(\cdot)$ is monotonically increasing, $m_{\rm RR}(i) \geq m_{\rm RR}$. With $i$ increasing, the convex combination gives more weight to the second term, which is larger than the first term. This explains the increasing of MSD at the first half of the cycle. With $i$ increasing further, $m_{\rm RR}(i)$ will decrease to the same level as  $m_{\rm RR}$. Hence, MSD at the second half of the cycle will decrease again.
The simulation result shows in the right plot of Fig.~\ref{fig:rr_factor} fits with the theoretical analysis for quadratic risks rather well. \vspace{-2mm}
\begin{figure}[htp]
	\centering
	\hspace{-2mm}\includegraphics[scale=0.265]{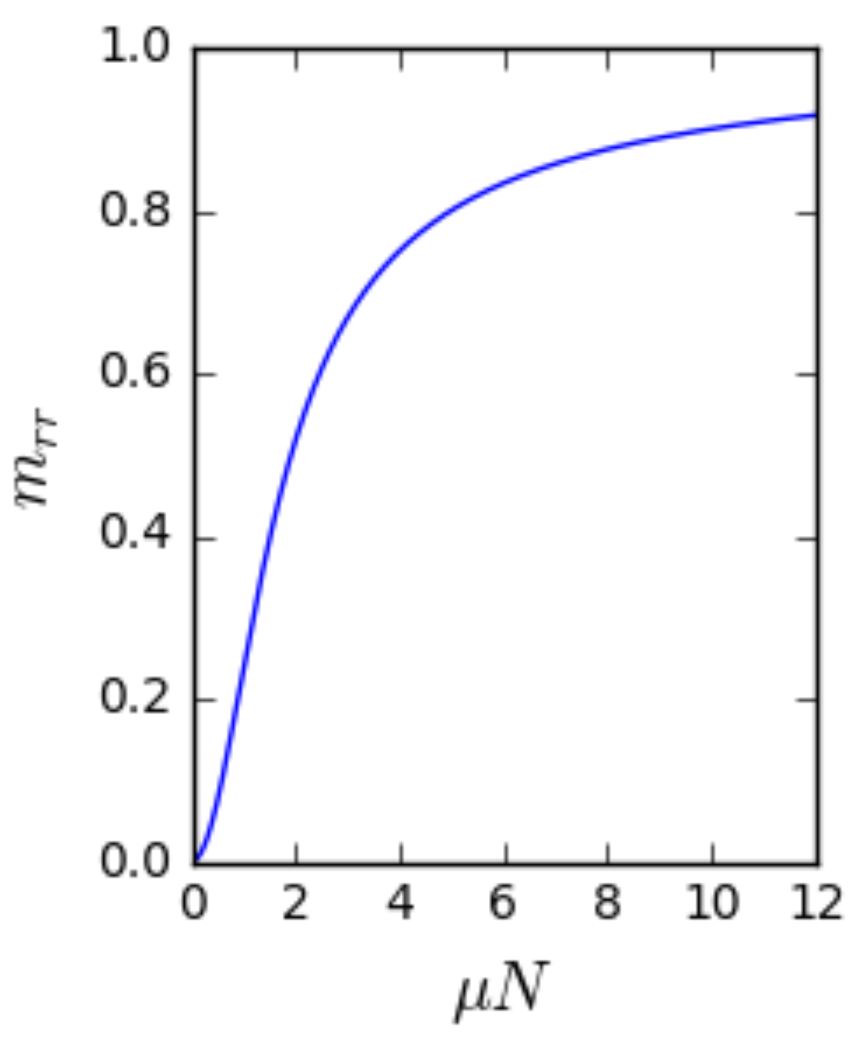}\vspace{-0mm}\hspace{-1.5mm}
	\includegraphics[scale=0.44]{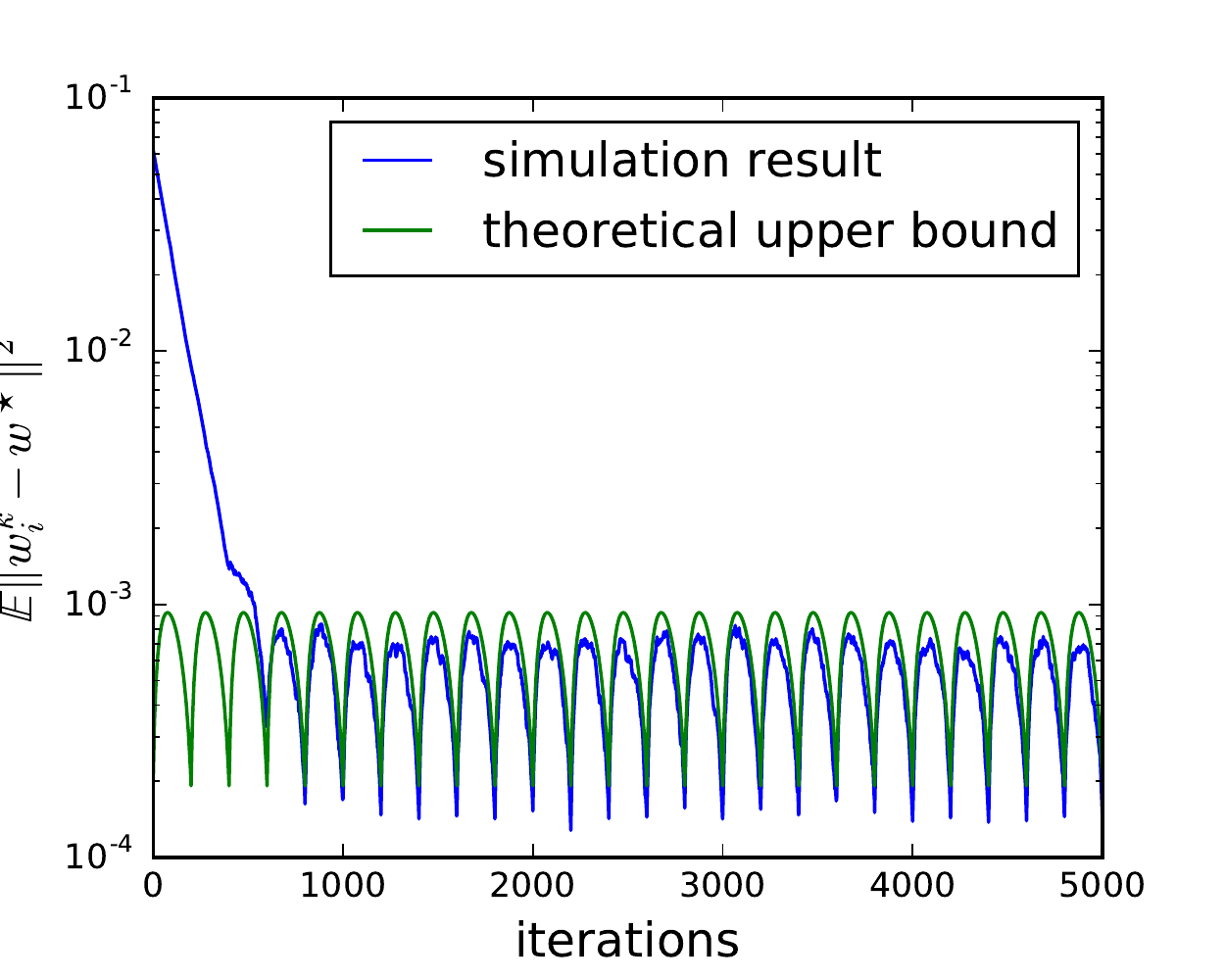}
	\caption{\small \color{black}Left: The curve of $m_{\rm RR}$ versus $\mu N$. Right: Mean-square-deviation performance of random reshuffling for a quadratic risk.}\vspace{-5mm}
	\label{fig:rr_factor}
\end{figure}

{
\begin{table*}[!htp]
	\caption{\color{black}\small Summary of the Results with Random Reshuffling versus Uniform Sampling with Replacement.\vspace*{-3mm}}\label{table.conclusion}
	\centering
	\footnotesize{\color{black}
		\begin{tabular}{|c||c|c|c|}
			\cline{1-4}
			{\color{white}$\Big($}&Uniform sampling\cite{yuan2016stochastic} \cellcolor[gray]{0.8}&	Random Reshuffling (Long-term or Quadratic) \cellcolor[gray]{0.8}&Random Reshuffling \cellcolor[gray]{0.8}\nn\\\hhline{====}
			{\color{white}$\Big($}Steady-state &$O(\mu)$&$O(\mu^3)$ --- Eq. \eqref{23gniox}&$O(\mu^2)$ --- Eq. \eqref{lemma.stability2}\\\cline{1-4}
			MSD$^{\rm epoch}$&$\frac{\mu}{2}\Tr(H^{-1}R_s)$& $\mu^2 \Tr\Big((I- (I-\mu H)^{2N})^{-1}R^{\prime \star}_s\Big)$ --- Eq. \eqref{eq.msd.reshuffle}
			&\eqref{eq.msd.reshuffle} + $O(\mu^2)$ --- Eq.\eqref{eq.mismatch} \\\cline{1-4}
			MSD$^{\rm epoch}$ (Hyperbolic)&$\displaystyle\frac{\mu}{2}\Tr\big( \Lambda^{-1} U\tran R_s^\star U \big)$&$\displaystyle\frac{\mu}{2}\Tr\Bigg(M_{\rm RR}\Lambda^{-1} U\tran R_s^\star U \Bigg)$ --- Eq. \eqref{msd.hyperbolic}&\eqref{msd.hyperbolic}+ $O(\mu^2)$ --- Eq.\eqref{eq.mismatch}\\\cline{1-4}
			MSD$^{\rm iteraion}$& $\frac{\mu}{2}\Tr(H^{-1}R_s)$& 
			\hspace{-4mm}$
			\begin{array}{ll}
				&(1-\mu\nu)^{2i}\mu^2\Tr\left(\big(I-(I-\mu H)^{2N}\big)^{-1}  R_{s}^{\prime \star}\right)\nn\\
				&\;\;	 +\Big(1\!-\!(1\!-\!\mu\nu)^{2i}\Big)\mu^2\Tr\left(\big(I\!-\!(I\!-\!\mu\nu)^{2i}\big)^{-1}  R_{s,i}^{\prime \star}\right)\nn
			\end{array}
			$--- Eq. \eqref{2389g.cs}
			&\eqref{2389g.cs} + $O(\mu^2)$ --- Eq.\eqref{eq.mismatch} \\\cline{1-4}
			MSD$^{\rm iteraion}$(Hyperbolic)& $\displaystyle\frac{\mu}{2}\Tr\big( \Lambda^{-1} U\tran R_s^\star U \big)$
			& \hspace{-8mm}$\begin{array}{ll}
			&e^{-2\mu i}\frac{\mu}{2}\Tr\Big(M_{RR}\Lambda^{-1} U\tran R_s^\star U \Big)\\
			&\;\;\;+ (1-e^{-2\mu i})\frac{\mu}{2}\Tr\Big(  M_{RR}(i)\Lambda^{-1} U\tran R_s^\star U \Big)
			\end{array}$
			 --- Eq. \eqref{23gdsdse}\hspace{-2mm}
			&\eqref{23gdsdse} + $O(\mu^2)$ --- Eq.\eqref{eq.mismatch} \\\cline{1-4}
			{\color{white}$\Big($}Infinite data&$\frac{\mu}{2}\Tr\big(H^{-1}R_s\big)$&$\frac{\mu}{2}\Tr(H^{-1}R_s)$--- Eq. \eqref{msd.infty}&$\frac{\mu}{2}\Tr(H^{-1}R_s)$ \\\cline{1-4}
			{\color{white}$\Big($}Periodic &No& Yes --- Eq. \eqref{2389g.cs} & Yes\\
			\cline{1-4}
		\end{tabular}\vspace{-3mm}
	}
\end{table*}
}
\subsection{Hyperbolic Representation for MSD in Long-term Model}\vspace{-1mm} \label{subsec.hyper}
Motivated by the result for the quadratic risk case, we now derive a similar expression for the MSD more generally also in terms of a tanh function. First, we extend result \eqref{exp.approx} into a matrix version. Supposing $\Lambda$ is a positive diagonal matrix and $\mu$ is sufficiently small such that $I-\mu\Lambda$ is a stable matrix, we have\vspace{-1mm}
\eq{
	(I-\mu \Lambda)^N \approx e^{-\mu N \Lambda}
}
and \vspace{-0.3mm}
\eq{
	\sum_{i=0}^{N-1}(I-\mu \Lambda)^i =& \frac{1}{\mu} (I-\mu \Lambda)^ N\Lambda^{-1}
	\approx \frac{1}{\mu}e^{-\mu N \Lambda}\Lambda^{-1} \label{g892.g3}
}
It follows that
\eq{
	&\Tr\left(\!(I- (I\!-\!\mu H)^{2N})^{-1}\left(\sum_{i=0}^{N-1} (I-\mu H)^iR^\star_s(I-\mu H)^i\right)\right)\nn\\
	&\stackrel{\eqref{eq.weight.solution}}{=}\Tr\left(\sum_{i=0}^{N-1} \sum_{k=0}^\infty(I-\mu H)^i(I-\mu H)^{2kN} (I-\mu H)^i R_s^\star\right)\nn\\
	&= \Tr\left( \sum_{k=0}^\infty\sum_{i=0}^{N-1}(I-\mu H)^{2(kN+i)}  R_s^\star\right)\nn\\
	&\stackrel{(a)}{=} \Tr\left(\sum_{j=0}^{\infty} (I-\mu H)^{2j}  R_s^\star\right)\nn\\
	&= \Tr\left((I-(I-\mu H)^2)^{-1} R_s^\star \right)\nn\\
	&\approx \frac{1}{2\mu}\Tr (H^{-1}R_s^\star)=\frac{1}{2\mu} \Tr(\Lambda^{-1} U\tran R_s^\star U )
}
where in step (a) we used the fact that  $k N + i$ is the $N$-modular representation of all integer numbers. To shorten the notation, we let:\vspace{-1mm}
\eq{
	\tau \define \mu N
}
Next, for the second part of \eqref{eq.msd.reshuffle}:
\eq{
	&\Tr\left(\!\!(I- (I\!-\!\mu H)^{2N})^{-1}\Big[\sum_{i=0}^{N-1}(I\!-\!\mu H)^i\Big] R^{ \star}_s \Big[\sum_{i=0}^{N-1}(I\!-\!\mu H)^i\Big]\!\right)\nn\\
	&\stackrel{(a)}{=}\!\frac{1}{\mu^2}\Tr\Big(\!(I\!-\! e^{\!-2\tau\Lambda})^{-1}(I\!-\!e^{\!-\tau \Lambda}) \Lambda^{-1} U\tran R_s^\star U \Lambda^{-1} (I\!-\!e^{\!-\tau \Lambda})\!\Big)\nn\\
	&= \frac{1}{\mu^2}\Tr\Big(\!\Lambda^{-1}(I\!-\!e^{\!-\tau \Lambda})(I\!-\!\mu e^{\!-2\tau\Lambda})^{-1}(I\!-\!e^{\!-\tau \Lambda}) \Lambda^{-1} U\tran R_s^\star U  \!\Big)\nn\\
	&\stackrel{(b)}{=}\frac{1}{\mu^2}\Tr\Big(\Lambda^{-1}(I+ e^{-\tau\Lambda})^{-1}(I-e^{-\tau \Lambda}) \Lambda^{-1} U\tran R_s^\star U \Big)
	\nn\\
	&= \frac{1}{\mu^2}\Tr\Big(\Lambda^{-1}\tanh(\tau\Lambda/2) \Lambda^{-1} UR_s^\star U\tran  \Big)\nn\\
	&\stackrel{}{=} \frac{N}{2\mu}\Tr\Big(2\tau^{-1}\Lambda^{-1}\tanh(\tau\Lambda/2) \Lambda^{-1} UR_s^\star U\tran  \Big)
}
where step (a) replaces $H$ by its eigendecomposition and uses \eqref{g892.g3}, while  step (b) exploits the fact that
\eq{
	I-e^{-2\tau \Lambda}  = (I+e^{-\tau \Lambda})(I-e^{-\tau \Lambda})
}
Moreover, the tanh notation refers to
\eq{
\tanh{\Lambda} = {\rm diag}\{\tanh(\Lambda_{1,1}), \cdots, \tanh(\Lambda_{M,M})\}
}
Combining the above two results gives
\eq{
&\hspace{-2mm}{\rm MSD}_{\rm RR}^{\rm lt} \nn\\
&= \frac{\mu}{2}\Tr\Bigg( \underbrace{ \frac{N}{N-1}\Big[I - \frac{2}{\mu N}\Lambda^{-1}\tanh\Big(\frac{\mu N}{2}\Lambda\Big)\Big]}_{\define M_{RR}}\Lambda^{-1} U\tran R_s^\star U \Bigg)\raisetag{4mm}\label{msd.hyperbolic}
}
Compared with the  uniform sampling case:
\eq{
	{\rm MSD}_{\rm US} =& \frac{\mu}{2}\Tr\big( H^{-1}R_s^\star \big)
	=\frac{\mu}{2}\Tr\big( \Lambda^{-1} U\tran R_s^\star U \big)\vspace{-2mm} \label{eq.msd.hyperbolic}
}
Now, it is clear that the diagonal matrix factor $M_{\rm RR}$ serves the same purpose  as $m_{\rm RR}$. Each entry of this factor matrix captures the improvement of random reshuffling over uniform sampling. Lastly, we focus on the order of expression \eqref{eq.msd.hyperbolic}. We know from the Taylor's expansion that
\eq{
	1 - \frac{1}{x}\tanh(x) = O(x^2)	
}
We conclude that 
\eq{
	M_{\rm RR} = O(\mu^2 N^2)\,\Longrightarrow\, {\rm MSD}_{\rm RR}^{\rm lt} = O(\mu^3)
	\label{23gniox}
}
that confirms the observation of $O(\mu^3)$ in the Fig.~\ref{fig:sgd_vs_rr_1en3}.
\vspace{-0mm}

{\color{black} Lastly, similar to the derivation for the quadratic case \eqref{g23g.xc}--\eqref{x2.dsy3g}, we can establish the hyperbolic representation of MSD for general case at all iterations:\small
	\eq{
		&\hspace{-5mm}\lim_{k\to\infty}\Ex\|\widetilde{\w}_{i}^{\prime k}\|^2\nn\\[-2mm]
		\approx&
		e^{-2\mu i}\frac{\mu}{2}\Tr\Bigg(\frac{N}{N-1}\Big[I - \frac{2}{\mu N}\Lambda^{-1}\tanh\Big(\frac{\mu N}{2}\Lambda\Big)\Big]\Lambda^{-1} U\tran R_s^\star U \Bigg)\nn\\
		&\; +\!  (1\! -\! e^{\! -2\mu i})\frac{\mu}{2}\Tr\Bigg(\!\! \underbrace{ \frac{N}{N\! -\! 1}\Big[I\!  -\!  \frac{2}{\mu N}\Lambda^{-1}\! \tanh\Big(\frac{\mu i}{2}\Lambda\Big)\Big]}_{\define M_{RR}(i)}\! \Lambda^{\! -\! 1} U\tran\!  R_s^\star U\!  \Bigg)\label{23gdsdse}\raisetag{4mm}
	}
}\normalsize
\vspace{-10mm}
\subsection{Infinite-Horizon Case}
In this work, we are mostly interested in the finite-data case, where the data size is $N$. The results so far are based on this assumption. However, it is inspiring though to see how the performance result would simplify if we allow $N$ to grow to infinity. In that case, we get
\begin{align}
\lim_{N\to\infty} {\rm MSD}^{\rm lt}_{\rm RR} =&\; \mu^2\lim_{N\to\infty}\Tr\left((I- (I-\mu H)^{2N})^{-1}R^{\prime\star}_s\right)
\nn\\
=&\;\mu^2\lim_{N\to\infty}\Tr\left(R^{\prime\star}_s\right)\label{eq.80}
\end{align}
since for sufficiently small $\mu$, the matrix $I-\mu H$ is stable. Moreover, observe further that:
\eq{\lim_{N\to\infty} \sum_{i=0}^{N-1}(I-\mu H)^{2i} =\;& \Big(I - (I-\mu H)^2\Big)^{-1}\nn\\
	=\;&\frac{1}{2\mu}H^{-1}\big(I - \mu H/2\big)^{-1}\nn\\
	=\;&\frac{1}{2\mu} H^{-1} + O(1)	
}
{\color{black} where $O(1)$ represents a matrix where all entries are $O(1)$.}
Hence, 
\begin{align}
\lim_{N\to\infty}\Tr\left(R^{\prime \star}_s\right) =&\; \Tr\left(\lim_{N\to\infty}R^{\prime \star}_s\right)\nn\\
=&\;\Tr\left(\lim_{N\to\infty} \sum_{i=0}^N(I-\mu H)^iR^\star_s(I-\mu H)^i \right)\nn\\
=&\;\Tr\left( \lim_{N\to\infty}\sum_{i=0}^N(I-\mu H)^{2i}R^\star_s \right)\label{98hj9gh23}\\
=&\; \frac{1}{2\mu} \Tr (H^{-1}R^\star_s) +O(1) \label{89g32.}
\end{align}
Substituting this result back into (\ref{eq.80}), we establish:
\be
\lim_{N\to\infty} {\rm MSD}_{\rm RR}^{\rm lt} = \frac{\mu}{2}\Tr (H^{-1}R_s) + O(\mu^2)
\label{msd.infty}\textsl{}\ee
which is exactly the same expression we have in the streaming data case\cite{sayed2014adaptation}. 
{\color{black}
	If we examine the hyperbolic approximation of MSD, performance is proportional to ${\rm tanh}(\mu N)$, which implies the performance will degrade with $\mu N$ but it will saturate if $\mu N$ keeps increasing. Equation \eqref{msd.infty} shows that the limit value is the same as the uniform sampling case. 
}


%

\section{Concluding Remarks}
In conclusion, this work studies the performance of stochastic gradient implementations under random reshuffling and provides a detailed analytical justification  for the improved performance of these implementations over uniform sampling. The work focuses on constant step-size adaptation, where the
agent is continuously learning.  The analysis establishes analytically that random
reshuffling outperforms uniform sampling by showing 
that iterates approach a smaller neighborhood of size $O(\mu^2)$
around the minimizer rather than $O(\mu)$. Simulation results
illustrate the theoretical findings. 
We also summarize the conclusions in Table \ref{table.conclusion}.

\appendices
\section{Proof of Theorem \ref{lemma.start.pont}}\label{proof.theorem.1}
{\color{black} Note first that
	\eq{
		\w_0^{k+1} \!\define& \w_{N}^k\nn\\
		\stackrel{\eqref{alg.rr}}{=}\;&\w_{N-1}^k - \mu\grad_w Q(\w_{N-1}^k; x_{\bsigma^k(N)}) \nn\\[-1mm]
		& \vdots\nn\\[-2mm]
		=\;& \w_{0}^k -\mu \sum_{i=1}^N\grad_w Q(\w_{i-1}^k; x_{\bsigma^k(i)})\nn\\
		\stackrel{(\ref{grad.define})}{=}\;& \w_{0}^k -\mu N \nabla_w J(\w^k_0) \label{gewg.gwe}\\
		&\;- \mu \sum_{i=1}^N\underbrace{\big( \grad_w Q(\w_{i-1}^k; x_{\bsigma^k(i)}) -
			\grad_w Q(\w_{0}^k; x_{\bsigma^k(i)})\big)}_{\define g_{\bsigma^k(i)}(\w^k_{i-1})} \nn
	}
	where we denote by $g_{\bsigma^k(i)}(\w^k_{i-1})$ the incremental gradient noise which is the mismatch between the gradient approximations evaluated at  $\w^k_0$ and $\w^k_{i-1}$.
	Next, we introduce the error vector: \eq{\widetilde{\w}_0^k \define  w^\star - \w_0^k}
	and let  $0~<~t~<~1$ be any scalar that we will specify further below. Subtracting $w^\star$ from both sides of (\ref{gewg.gwe}) and squaring, we get:
	{\color{black}
		\eq{
			&\hspace{-5mm}\|\tw_0^{k+1}\|^2 \nn\\
			=& \left\|\tw_{0}^k + \mu N \nabla_w J(\w^k_0)+\mu\sum_{i=1}^Ng_{\bsigma^k(i)}(\w^k_{i-1})\right\|^2\nn\\
			\stackrel{(a)}{\leq}& \frac{1}{t}\|\tw_{0}^k + \mu N \nabla_w J(\w^k_0) \|^2 +\frac{\mu^2}{1-t}\left\|\sum_{i=1}^Ng_{\bsigma^k(i)}(\w^k_{i-1})\right\|^2\nn\\
			\stackrel{(b)}{\leq}& \frac{1}{t}\left\|\tw_{0}^k \!+\! \mu N \nabla_w J(\w^k_0) \right\|^2  \!\!+\!\frac{\mu^2N}{1-t}\left(\sum_{i=1}^N\left\|g_{\bsigma^k(i)}(\w^k_{i-1})\right\|^2\right) \label{2h89jbg.sbd}
		}
		where step (a) exploits Jensen's inequality:
		\eq{
			\|a+b\|^2 = \left\|\frac{t}{t} a+ \frac{1-t}{1-t}b\right\|^2 \leq \frac{1}{t}\| a \| + \frac{1}{1-t}\|b\|^2  
		}
	}
	and step (b) uses the fact that:
	\eq{
		\left\|\sum_{i=1}^N x_i\right\|^2=N^2\left\|\sum_{i=1}^N\frac{1}{N} x_i\right\|^2\leq N\sum_{i=1}^N\left\|x_i\right\|^2 \label{eq.jensen}
	}
	We show in Appendix \ref{app.incr.noise} that the rightmost term in (\ref{2h89jbg.sbd}) can be bounded by:
	\eq{
		\sum_{i=1}^N\left\|g_{\bsigma^k(i)}(\w^k_{i-1})\right\|^2 
		&\leq \frac{\mu^2 \delta^2N^3}{1-2\mu^2\delta^2N^2}\left(2\delta^2\|\tw_0^k\|^2  + \cK \right) \label{fwehig.we}
	}
	while for the first term in \eqref{2h89jbg.sbd} we have
	{\color{black}
		\eq{
			&\hspace{-4mm}\left\|\tw^k_0 + \mu N \nabla J(\w_0^k)\right\|^2\nn\\[1mm]
			=& \|\tw^k_0\|^2 + \mu^2 N^2\| \nabla J(\w_0^k)\|^2+ 2 \mu N(\tw^k_0)\tran \nabla J(\w_0^k)\nn\\[1mm]
			\leq& \Big(1-2\mu N\frac{\nu\delta}{\delta+\nu}\Big)\|\tw^k_0\|^2 + \mu N(\mu N-\frac{2}{\delta+\nu}) \| \nabla J(\w_0^k)\|^2 \nn\\[-1mm]
			\label{f9j30.3x}
		}
		where in the first inequality we exploit the co-coercivity inequality\cite{nesterov2013introductory} that 
		\eq{
			&(\nabla J(x) - \nabla J(y))\tran (x-y)\nn\\
			&\;\;\;\;\;\geq \frac{\nu\delta}{  \delta+\nu}\|x-y\|^2+\frac{1}{\delta+\nu}\|\nabla J(x)-\nabla J(y)\|^2
		}
		Next we require the step size to satisfy
		\eq{
			\mu \leq \frac{2}{(\delta+\nu)N} \label{stepsize.condition1}
		}
		Then,  the coefficient of the last term in \eqref{f9j30.3x} is negative. Combining with the strongly convexity property $\| \nabla J(\w_0^k) - \nabla J(w^\star)\| \geq \nu\|\tw_{0}^k\|$, we have
		\eq{
			&\hspace{-4mm}\left\|\tw^k_0 + \mu N \nabla J(\w_0^k)\right\|^2\nn\\[1mm]
			\leq&\Big(1-2\mu N\frac{\nu\delta}{\delta+\nu}\Big)\|\tw^k_0\|^2 + \mu N\nu^2(\mu N-\frac{2}{\delta+\nu}) \|\tw^k_0\|^2\nn\\
			=&\Big(1-\mu \nu N)^2\|\tw^k_0\|^2	\label{f9j30.3}
		}
		Combining \eqref{fwehig.we} and \eqref{f9j30.3}, we establish:
		\eq{
			\|\tw^{k+1}_0\|^2\leq& \frac{1}{t} \left(1- \mu N \nu \right)^2 \|\tw_{0}^k\|^2\nn\\
			&\;\;{} + \frac{\mu^2N}{1-t} \frac{\mu^2 \delta^2 N^3}{1-2\mu^2\delta^2N^2}\left(2\delta^2\|\tw_0^k\|^2  + \mathcal{K}\right)
		}
		\noindent We are free to choose $t\in(0,1)$. Thus, let $t=1-\mu N \nu$.  Then, we conclude that
		\eq{
			\|\tw_0^{k+1}\|^2 
			\leq& \Bigg(1-\mu N \nu + \frac{2\mu^3\delta^4N^3}{\nu(1-2\mu^2\delta^2N^2)}\Bigg)\|\tw^k_0\|^2\nn\\
			&\;\;+\frac{\mu^3\delta^2N^3\mathcal{K}}{\nu(1-2\mu^2\delta^2N^2)}  \label{g32jo.sgw}
		}
		If we assume $\mu$ is sufficiently small such that
		\eq{\label{mu-2}
			1-2\mu^2\delta^2N^2 \ge \frac{1}{2},
		}
		then inequality \eqref{g32jo.sgw} becomes
		\eq{\label{xcnweu-2}
			\|\tw_0^{k+1}\|^2 \leq\;& \Big(1-\mu N \nu + \frac{4\mu^3\delta^4N^3}{\nu}\Big)\|\tw^k_0\|^2+\frac{2\mu^3\delta^2N^3}{\nu}\mathcal{K}.
		}
		If we further assume the step-size $\mu$ is sufficiently small such that
		\eq{\label{mu-3}
			1-\mu N \nu + \frac{4\mu^3\delta^4N^3}{\nu} \le 1-\frac{1}{2} \mu N \nu
		}
	}
	then inequality \eqref{xcnweu-2} becomes
	\eq{
		\|\tw_0^{k+1}\|^2 \leq\;& \left(1-\frac{1}{2}\mu N \nu \right) \|\tw^k_0\|^2+\frac{2\mu^3\delta^2N^3}{\nu}\mathcal{K}  
	}
	Iterating over $k$, we have
	\eq{
		\|\tw_0^{k+1}\|^2\leq\;& \left(1-\frac{1}{2}\mu N \nu \right)^k \|\tw^0_0\|^2\nn\\
		&\;\;{}+\left(\frac{2\mu^3\delta^2N^3}{\nu}\mathcal{K}\right)\sum_{j=1}^{k}\left(1-\frac{1}{2}\mu N \nu \right)^j  \nnb
		\le\, & \left(1-\frac{1}{2}\mu N \nu \right)^k \|\tw^0_0\|^2+\frac{4\mu^2\delta^2N^2}{\nu^2}\mathcal{K}\label{sdnsdn}
	}
	By taking expectations {\color{black} with respect to the filtration, i.e., the collection of past information,} on both sides, we have
	\eq{
		\bE\|\tw_0^{k+1}\|^2 \leq\;&  \left(1-\frac{1}{2}\mu N \nu \right)^k \bE \|\tw^0_0\|^2+\frac{4\mu^2\delta^2N^2}{\nu^2}\mathcal{K}, \label{13g,sdfgew}
	}
	which implies that
	\eq{
		\limsup_{k\to \infty} \bE\|\tw_0^{k}\|^2  = O(\mu^2)
	}
	Finally we find a sufficient range for $\mu$ for stability. To satisfy \eqref{stepsize.condition1}, \eqref{mu-2} and \eqref{mu-3}, it is enough to set $\mu$ as
	\eq{\color{black}
		\mu &\color{black} \le \min \left\{ \frac{2}{(\delta+\nu)N} , \frac{1}{2\delta N}, \frac{\nu}{\sqrt{8}\delta^2 N} \right\} < \frac{\nu}{3\delta^2 N}.
	}
	The argument in this derivation provides a self-contained proof for the convergence result (\ref{lemma.stability2}), which generalizes the approach from \cite{ying2017rr}. There, the bound (\ref{lemma.stability2}) was derived from an intermediate property (23) in \cite{ying2017rr}, which does not always hold. Here, the same result is re-derived and shown to hold irrespective of this property. Consequently, we are now able to obtain Lemma 1 from \cite{ying2017rr} as a corollary to our current result, as shown next. 
}

\section{Derivation of \eqref{fwehig.we}}\label{app.incr.noise}
Indeed, from  Lipschitz continuity of the gradients, we have
\eq{
	\sum_{i=1}^N\left\|g_{\bsigma^k(i)}(\w^k_{i-1})\right\|^2 \leq\,& \sum_{i=1}^N\delta^2\left\|\w^k_{i-1}-\w^k_0\right\|^2 \nn\\
	=\,&\delta^2\sum_{i=1}^N\left\|\sum_{j=1}^{i-1}(\w_{j}^k - \w_{j-1}^k)\right\|^2\nn\\
	\stackrel{\eqref{eq.jensen}}{\leq}&\delta^2\sum_{i=1}^N(i-1)\sum_{j=1}^{i-1}\|\w_{j}^k - \w_{j-1}^k\|^2
}
Using the equivalence relation \eq{
	\sum_{i=1}^N\sum_{j=1}^{i-1} a_{ij}\equiv \sum_{j=1}^{N-1}\sum_{i=j+1}^N  a_{ij} \label{eq.equival.sum}
}
we obtain
\eq{
	\sum_{i=1}^N\left\|g_{\bsigma^k(i)}(\w^k_{i-1})\right\|^2\leq& \delta^2\sum_{j=1}^N\sum_{i=j+1}^{N}(i-1)\|\w_{j}^k - \w_{j-1}^k\|^2\nn\\
	\leq&\frac{\delta^2N^2}{2}\sum_{j=1}^N\|\w_{j}^k - \w_{j-1}^k\|^2 \label{gwioeg}
}
\noindent where in the second inequality we used the fact that
\eq{
	\sum_{i=j+1}^N(i-1)\leq \sum_{i=1}^N(i-1)=\frac{N(N-1)}{2} \leq\frac{N^2}{2},\;\;j=1,2,\ldots,N \label{eq.summation}
}
We can recursively bound the difference  terms in \eqref{gwioeg} as follows. From \eqref{alg.rr}, we have
\eq{
	&\hspace{-5mm}\|\w_{j}^k - \w_{j-1}^k\|^2 \nn\\
	=& \mu^2 \|\nabla_w Q(w_{j-1}; x_{\bsigma^k{(j)}})\|^2\nn\\
	\leq& 2\mu^2 \|\nabla_w Q(w_{j-1}; x_{\bsigma^k{(j)}}) - \nabla_w Q(w^\star; x_{\bsigma^k(j)})\|^2\nn\\
	&\;\;{}+2\mu^2 \|\nabla_w Q(w^\star; x_{\bsigma^k(j)})\|^2\nn\\
	\leq& 2\mu^2\delta^2\|\tw_{j-1}^k\|^2 + 2\mu^2 \|\nabla_w Q(w^\star; x_{\bsigma^k(j)})\|^2\nn\\
	\leq& 4\mu^2\delta^2\|\tw_0^k\|^2 +  4\mu^2\delta^2\|\w_{j-1}^k -\w_{0}^k\|^2\nn\\
	&\;\;{}+2\mu^2 \|\nabla_w Q(w^\star; x_{\bsigma^k(j)})\|^2
}
Summing over $j$:
\eq{
	&\hspace{-5mm}\sum_{j=1}^{N}\|\w_{j}^k - \w_{j-1}^k\|^2 \nn\\[-2mm]
	\stackrel{\eqref{gradient.noise}}{\leq}\;&	4\mu^2\delta^2N\|\tw_0^k\|^2  + 2\mu^2N\mathcal{K} +  4\mu^2\delta^2\sum_{j=1}^N\|\w_{j-1}^k -\w_{0}^k\|^2\nn\\
	=\;&	4\mu^2\delta^2N\|\tw_0^k\|^2  \!+\! 2\mu^2N\mathcal{K}\!+ \! 4\mu^2\delta^2\sum_{j=1}^N\left\| \sum_{i=1}^{j-1}(\w_{i}^k -\w_{i-1}^k)\right\|^2\nn\\
	\stackrel{\eqref{eq.jensen}}{=}\;&\!	4\mu^2\delta^2N\|\tw_0^k\|^2 \! +\! 2\mu^2N\mathcal{K} \nn\\
	&\;\;+  4\mu^2\delta^2\sum_{j=1}^N\sum_{i=1}^{j-1}(j-1)\| \w_{i}^k \!-\!\w_{i-1}^k\|^2\nn\\
	\stackrel{\eqref{eq.equival.sum}}{=}\;&\!	4\mu^2\delta^2N\|\tw_0^k\|^2  + 2\mu^2N\mathcal{K} \nn\\
	&\;\;+  4\mu^2\delta^2\sum_{i=1}^{N-1}\sum_{j=i+1}^{N}(j-1)\| \w_{i}^k -\w_{i-1}^k\|^2\nn\\
	\stackrel{\eqref{eq.summation}}{\leq}\;&\!	4\mu^2\delta^2N\|\tw_0^k\|^2  + 2\mu^2N\mathcal{K} +  2\mu^2\delta^2 N^2\sum_{j=1}^{N-1}\| \w_{j}^k -\w_{j-1}^k\|^2\nn\\
	\leq\;&	\!4\mu^2\delta^2N\|\tw_0^k\|^2  + 2\mu^2N\mathcal{K} +  2\mu^2\delta^2 N^2\sum_{j=1}^{N}\| \w_{j}^k -\w_{j-1}^k\|^2
}
Rearranging the terms, we get
\eq{
	(1\!-\!2\mu^2\delta^2N^2)\sum_{j=1}^{N}\|\w_{j}^k - \w_{j-1}^k\|^2 \leq 4\mu^2\delta^2N\|\tw_0^k\|^2  \!+\! 2\mu^2N\mathcal{K}\raisetag{2mm}
}
After substituting into \eqref{gwioeg} and simplifying, we have \eqref{fwehig.we}.

%
%
%

\section{Proof of Corollary \ref{lemma.corollary.1}}\label{proof.corollary.1}
We have
\eq{
	\Ex \|\tw^k_i\|^2 \leq& 2 \Ex\|\w^k_i - \w^k_0\|^2 + 2\Ex \|\tw^k_0\|^2\nn\\
	\leq& \color{black}2 i\sum_{j=0}^{i-1}\Ex\|\w^k_{j+1} - \w^k_j\|^2 + 2\Ex \|\tw^k_0\|^2\nn\\
	\leq& 2 i\sum_{j=0}^{i-1}\Ex\|\nabla_w Q(\w^k_j; x_{\bsigma^k(j)})\|^2 + 2\Ex \|\tw^k_0\|^2\nn\\
	\leq& 2 \mu^2\delta^2i\sum_{j=0}^{i-1}\Ex\|\tw^k_j\|^2 + 2\Ex \|\tw^k_0\|^2
}
Summing over $i$;
\eq{
	&\hspace{-5mm}\sum_{i=1}^{N-1}\Ex \|\tw^k_i\|^2\nn\\
	\leq& 2 \mu^2\delta^2\sum_{i=1}^{N-1}\sum_{j=0}^{i-1}i\Ex\|\tw^k_j\|^2 + 2N\Ex \|\tw^k_0\|^2\nn\\
	=& 2 \mu^2\delta^2\sum_{j=0}^{N-1}\sum_{i=j+1}^{N-1}i\Ex\|\tw^k_j\|^2 + 2N\Ex \|\tw^k_0\|^2\nn\\
	\leq& \mu^2\delta^2N^2\sum_{j=0}^{N-1}\Ex\|\tw^k_j\|^2 + 2N\Ex \|\tw^k_0\|^2\nn\\
	=& \mu^2\delta^2N^2\sum_{j=1}^{N-1}\Ex\|\tw^k_j\|^2 + (2N+\mu^2\delta^2N^2)\Ex \|\tw^k_0\|^2
}
Rearranging terms, we get
\eq{
	\sum_{i=1}^{N-1}\Ex \|\tw^k_i\|^2\leq&\frac{2N+\mu^2\delta^2N^2}{1-\mu^2\delta^2N^2}\Ex \|\tw^k_0\|^2
}
Let $k\to\infty$, then
\eq{
	\limsup_{k\to\infty}\sum_{i=1}^{N-1}\Ex \|\tw^k_i\|^2 = O(\mu^2)
}
Noting that every term in the summation is non-negative, we conclude that for all \( j\):
\eq{
	\limsup_{k\to\infty}\Ex \|\tw^k_j\|^2 \le \limsup_{k\to\infty}\sum_{i=1}^{N-1}\Ex \|\tw^k_i\|^2 = O(\mu^2)
}\vspace{-3mm}

{\color{black}
	\section{Proof of corollary \ref{lemma.corollary.2}}\label{proof.corollary.2}\vspace{-1mm}
	For completeness, it is easy to modify our derivation to arrive at a similar conclusion to  \cite{Gurbuzbalaban2015} for the diminishing step-size scenario.
	
	Observe that inequality (\ref{13g,sdfgew}) continues to hold for decaying step-sizes:
	\eq{
		\Ex\|\tw_0^{k+1}\|^2 \leq\;& \left(1-\frac{1}{2}\mu(k) N \nu \right) \Ex\|\tw^k_0\|^2+\frac{2\mu(k)^3\delta^2N^3}{\nu}\mathcal{K} \label{xdfe.gwe}
	}
	For simplicity, we only consider step-size sequences of the form:
	\eq{
		\mu(k) = \frac{c}{k+1}, \;\;\;\;k\geq 0
	}
	where $c$ is some positive constant. Then, we can exploit Chung's Lemma \cite{polyak1987introduction} or \cite[Lemma F.5]{sayed2014adaptation}
	%
	to conclude  that the convergence rate of $\Ex\|\tw_0^{k+1}\|^2$ is $O(1/k^2)$. The relationship between the number of epoch $k$ and iteration $i$ is linear. Therefore, it also follows that the convergence rate is $O(1/i^2)$.
}
\vspace{-4mm}

\section{Proof of Lemma \ref{lemma.2}} \label{app.noise}\vspace{-1mm}
We employ mathematical induction. First, it is easy to verify that $f(1;X,\beta) = {\rm Var}(X)$. Now, assuming (\ref{eq.lemma2}) is correct for case $n$, we consider  case $n+1$:
\begin{align}
f(n+1;X, \beta)
&= \Ex \Bigg\|\sum_{j=1}^{n+1}\beta^{n+1-j} x_{\bsigma(j)}\Bigg\|^2\nn\\
&= \Ex \Bigg\|\beta\sum_{j=1}^{n}\beta^{n-j} x_{\bsigma(j)}+x_{\bsigma(n+1)}\Bigg\|^2\nn\\
&= \beta^2\Ex\Bigg \|\sum_{j=1}^{n}\beta^{n-j} x_{\bsigma(j)}\Bigg\|^2+\Ex \|x_{\bsigma(n+1)}\|^2\nn\\
&\hspace{7mm} + 2\beta\Ex \Bigg(\sum_{j=1}^{n}\beta^{n-j} x_{\bsigma(j)}\Bigg)\tran x_{\bsigma(n+1)} \label{gu9ng.xd}
\end{align}
From the uniform random reshuffling property \eqref{prop2}, we know that:
\eq{
	\Ex\|x_{\bsigma(n+1)}\|^2 = {\rm Var}(x) \label{zdxc.CS}
}
For the cross terms, we exploit the law of total expectation\cite{durrett2010probability}:
\eq{
	&\hspace{-2mm}\Ex\Bigg(\sum_{j=1}^{n}\beta^{n-j} x_{\bsigma(j)}\Bigg)\tran x_{\bsigma(n+1)}\nn\\
	&=\Ex_{\bsigma(1:n)}\!\!\left[\!\Ex_{\bsigma(n+1)}\Bigg(\sum_{j=1}^{n}\beta^{n-j} x_{\bsigma(j)}\Bigg)\tran x_{\bsigma(n+1)}\Big|\,\bsigma(1:n)\!\right]\nn\\
	&\stackrel{\eqref{prop3}}{=}\Ex_{\bsigma(1:n)}\left[\Bigg(\sum_{j=1}^{n}\beta^{n-j} x_{\bsigma(j)}\Bigg)\tran \left(\frac{1}{N-n}\sum_{j\notin \bsigma(1:n)}x_j\right)\right]\nn\\
	&=-\frac{1}{N-n}\Ex_{\bsigma(1:n)}\left[\Bigg(\sum_{j=1}^{n}\beta^{n-j} x_{\bsigma(j)}\Bigg)\tran\sum_{j=1}^{n}x_{\bsigma(j)}\right]\nn\\
	&=-\frac{1}{N-n}\Ex_{\bsigma(1:n)}\sum_{j=1}^{n}\beta^{n-j}\|x_{\bsigma(j)}\|^2 \nn\\
	&\;\;\;\;\;\;-\frac{1}{N-n}\Ex_{\bsigma(1:n)}\sum_{i=1}^{n}\beta^{n-i} \left(\sum_{j=1, j\neq i}^{n}x\tran_{\bsigma(i)}x_{\bsigma(j)}\right) \label{gwe090jg2}
}
Without loss of  generality, we assume $i<j$ in the following argument. If $i>j$, exchanging the place of $x_{\bsigma(i)}$ and $x_{\bsigma(j)}$ leads to the same conclusion:
\begin{align}
\Ex_{\bsigma(1:n)}\big[x\tran_{\bsigma(i)}x_{\bsigma(j)}\big] =\,&\,
\Ex_{\bsigma(i),\bsigma(j)}\big[x\tran_{\bsigma(i)}x_{\bsigma(j)}\big]\nn\\
=\,&\, \Ex_{\bsigma(i)}\left\{ x\tran_{\bsigma(i)}\Ex_{\bsigma(j)}[ x_{\bsigma(j)}\,|\, \bsigma(i)]\right\}\nn\\
\stackrel{(\ref{prop3})}{=}&\, -\frac{1}{N-1} \Ex_{\bsigma(i)}\|x_{\bsigma(i)}\|^2\nn\\
=&\, -\frac{1}{N-1} {\rm Var}(X) \label{oijw.ni}
\end{align}
Substituting \eqref{oijw.ni} into \eqref{gwe090jg2}, we obtain:
\begin{align}
&\hspace{-5mm}\Ex\Bigg(\sum_{j=1}^{n}\beta^{n-j} x_{\bsigma(j)}\Bigg)\tran x_{\bsigma(n+1)}\nn\\
=& -\frac{1}{N-n}\left(\sum_{j=1}^{n}\beta^{n-j} -\sum_{j=1}^n\beta^{n-j}\frac{n-1}{N-1}\right){\rm Var}(X) \nn\\
=&\; -\frac{1}{N-1} \sum_{j=1}^n \beta^{j-1} {\rm Var}(X)\label{eq.63}
\end{align}
Combining \eqref{gu9ng.xd}, \eqref{zdxc.CS},  and \eqref{eq.63}, we get:
\begin{align}
&\hspace{-3mm}f(n+1;X, \beta)\nn\\
=& \beta^2 f(n;X, \beta) + {\rm Var}(X) - \frac{2}{N-1} \sum_{j=1}^n \beta^{j} {\rm Var}(X)\nn\\
=&\left(\beta^2\frac{(\sum_{i=0}^{n-1}\beta^{2i})N\!-\!(\sum_{i=0}^{n-1}\beta^i)^2}{N-1} \!+\! 1 \!-\! \frac{2\sum_{j=1}^n \beta^{j}}{N-1}\right) {\rm Var}(X)
\nn\\
=&\frac{(\sum_{i=1}^{n}\beta^{2i})N-(\sum_{i=1}^{n}\beta^i)^2 + (N-1) - 2\sum_{j=1}^n \beta^{j}}{N-1}  {\rm Var}(X)
\nn\\
=& \frac{(\sum_{i=0}^{n}\beta^{2i})N-(\sum_{i=0}^{n}\beta^i)^2}{N-1}{\rm Var}(X)
\end{align}
Hence, we conclude that (\ref{eq.lemma2}) is valid.

Next, the proof of \eqref{eq.lemma2.2} is  similar. It is easy to verify that $F(1;X, B) = R_x$. Assuming (\ref{eq.lemma2.2}) is correct for case $n$, we consider  case $n+1$:
\begin{align}
&\hspace{-5mm}F(n+1; X, B) \nn\\
=\;&\Ex \!\left[\!\sum_{j=1}^{n}B^{n-j}x_{\bsigma(j)} + x_{\bsigma(n+1)}\!\right]\!\!\left[\sum_{j=1}^{n}x_{\bsigma(j)}\tran B^{n-j} +x_{\bsigma(n+1)}\!\right]\nn\\
=\,& B F(n;X,B) B + \Ex\sum_{j=1}^{n}B^{n-j}x_{\bsigma(j)} x_{\bsigma(n+1)}\tran\nn\\
&\;\; {}+\Ex\sum_{j=1}^{n}x_{\bsigma(n+1)} x_{\bsigma(j)}\tran B^{n-j} + R_s\nn\\
\stackrel{(\ref{prop3})}{=}&B F(n;X,B) B - \frac{1}{N-n}\Ex\sum_{j=1}^{n}\sum_{i=1}^{n}B^{n-j}\ x_{\bsigma(j)} x_{\bsigma(i)}\tran\nn\\
&\;\; {}-\frac{1}{N-n}\Ex\sum_{j=1}^{n}\sum_{i=1}^{n} x_{\bsigma(i)} x_{\bsigma(j)}\tran B^{n-j} +R_s
\nn\\
\stackrel{(a)}{=}\,&B F(n;X,B)B - \frac{1}{N-1}\sum_{j=1}^{n}B^{n-j}R_s\nn\\
&\;\; {}-\frac{1}{N-1}\sum_{j=1}^{n}R_s B^{n-j} + R_s \label{9jgh2.d}
\end{align}
where in the step (a) we use the same trick as (\ref{eq.63}):
\eq{
	&\hspace{-5mm}\Ex\sum_{j=1}^{n}\sum_{i=1}^{n}B^{n-j}x_{\bsigma(j)}x_{\bsigma(i)}\tran\nn\\
	=\,&\Ex\sum_{j=1}^{n}B^{n-j}x_{\bsigma(j)}x_{\bsigma(j)}\tran +
	\Ex \sum_{j=1}^{n}\sum_{i\neq j}B^{n-j}x_{\bsigma(j)}x_{\bsigma(i)}\tran\nn\\
	=\,&\sum_{j=1}^{n}B^{n-j}R_s - \frac{1}{N-1}\Ex\sum_{j=1}^{n}B^{n-j}x_{\bsigma(j)}x_{\bsigma(j)}
}
Now if we substitute the $F(n;X,B)$ according to (\ref{eq.lemma2.2}) into (\ref{9jgh2.d}), we will conclude that the format of (\ref{eq.lemma2.2}) is still valid for $F(n+1; X, B)$, which completes the proof.

\section{Proof of Theorem \ref{main.theorem}} \label{app.main.theorem}
We introduce the eigen-decomposition \cite{Horn03}
\eq{
	H = U\Lambda U\tran
}
where $U$ is orthogonal and $\Lambda$ is diagonal with positive entries. Transforming  (\ref{eq.long_term.rec}) into the eigenvector space of $H$, we obtain:
\begin{align}
U\tran \widetilde{\w}_0^{\prime k+1} = (I-\mu\Lambda)^N U\tran \widetilde{\w}_0^{\prime k} - \mu   U\tran s'(\w_0^k) \label{89g23.ge}
\end{align}
Let
\be
\bar \w_0^{k} \define U\tran \widetilde{\w}_0^{\prime k}
\ee
and introduce any positive-definite matrix $\Sigma$. Computing the weighted square norm of both sides of (\ref{89g23.ge}) and taking expectations we get
\begin{align}
\Ex \|\bar \w_0^{k+1}\|^2_{\Sigma} \stackrel{\eqref{noise.zero}}{=} \Ex \|(I-\mu\Lambda)^N\bar \w_0^{k}\|^2_{\Sigma} + \mu^2\Ex \| U\tran s'(\w_0^k)\|^2_\Sigma
\end{align}
where $\|x\|^2_\Sigma \define x\tran\Sigma x$ and we are free to choose $\Sigma$. The cross term is canceled thanks to property \eqref{noise.zero}.
Letting $k\to\infty$, we get:
\begin{align}
\lim_{k\to \infty} \Ex \|\bar{\w}_{0}^k\|^2_{\Sigma - (I-\mu \Lambda)^N\Sigma (I-\mu \Lambda)^N} = \lim_{k\to \infty}\mu^2\Ex \| U\tran s'(\w_0^k)\|^2_\Sigma \label{eq.limit}
\end{align}
To recover the mean-square-deviation $\Ex \|\bar{\w}_{0}^k\|^2$, we choose $\Sigma$ as the solution to the Lyapunov equation:
\be
\Sigma - (I-\mu \Lambda)^N\Sigma (I-\mu \Lambda)^N = I\label{8hj9g.3}
\ee
which is given by
\begin{align}
\Sigma^\star = \sum_{k=0}^\infty (I - \mu \Lambda)^{2N k}
=& \left(I - (I-\mu \Lambda)^{2N}\right)^{-1} \label{eq.weight.solution}
\end{align}
{\color{black} where we require $\mu<\frac{1}{\delta}$ for the stability of infinite series summation.}
The desired MSD is given by:
\begin{align}
{\rm MSD}^{\rm lt}_{\rm RR} \define\lim_{k\to \infty} \Ex \|\widetilde{\w}_{0}^{\prime k}\|^2
=\lim_{k\to \infty} \Ex \|\bar{\w}_{0}^k\|^2
\end{align}
and, hence, 
\begin{align}
&\hspace{-5mm}\lim_{k\to \infty} \Ex \|\bar{\w}_{0}^k\|^2
\nn\\
\stackrel{(\ref{eq.limit})}{=}&\,
\lim_{k\to \infty}\mu^2\Ex \| U\tran s'(\w^k_0)\|^2_{\Sigma^\star}\nn\\
=&\, \lim_{k\to \infty}\mu^2 \Tr\left(U\Sigma^\star U\tran \Ex s'(\w_0^k) s'(\w_0^k)\tran\right) \nn\\
=&\,\lim_{k\to \infty}\mu^2 \Tr\left(U\Sigma^\star U\tran \Ex s'(w^\star) s'(w^\star)\tran\right) +
\nn\\
&\;\;\lim_{k\to \infty}\mu^2 \Tr\left(U\Sigma^\star U\tran \Ex s'(\w_0^k)s'(\w_0^k)\tran - \Ex s'(w^\star) s'(w^\star)\tran\right)\nn\\
=& \mu^2 \Tr\left(U\Sigma^\star U\tran \Ex s'(w^\star) s'(w^\star)\tran\right) + O(\mu^4) \label{89jgws.ge}
\end{align}
The proof of last equality is provided in Appendix \ref{g892.gf3}.
Combining \eqref{eq.weight.solution} and the fact that $U$ is the eigenvector matrix of $H$, we get:
\eq{
	&\hspace{-5mm}{\rm MSD}^{\rm lt}_{\rm RR} \nn\\
	=\,&\mu^2 \Tr\left(U\sum_{k=0}^\infty (I - \mu \Lambda)^{2N k} U\tran \Ex s'(w^\star) s'(w^\star)\tran\right) + O(\mu^4)\nn\\
	=\,&\mu^2 \Tr\left(\sum_{k=0}^\infty (I - \mu H)^{2N k}  \Ex s'(w^\star) s'(w^\star)\tran\right) + O(\mu^4)\nn\\
	=\,&\mu^2 \Tr\left(\big(I-(I-\mu H)^{2N}\big)^{-1}  R_s^{\prime \star}\right) + O(\mu^4)\label{msd.lt}
}
{\color{black}
As for the convergence rate, we can follow the same argument in \cite[Chapter 4]{sayed2014adaptation} to get
\eq{
	\alpha \define  (1 - \mu \lambda_{\min}(H))^{2N} \approx 1 - 2\mu \lambda_{\min}(H) N
}
}
\vspace{-3mm}

{\color{black}
\section{Proof of Theorem \ref{theorem.iterations}}\label{app.theorem.iterations}
Using a similar approach to \eqref{eq.error_expand}, we have 
\eq{
	\widetilde{\w}_{i}^{\prime k}
	= (I-\mu H)^i\widetilde{\w}_{0}^{\prime k} - \mu\underbrace{\sum_{j=1}^{i}(I-\mu H)^{i-j} s_{\bsigma^k(i)}(\w_{0}^k)}_{\define s_i'(\w_0^k)} \label{eq.long_term.i}\raisetag{4mm}
}
where 
\eq{
	\Ex[s_i'(\w_0^k) | \w_0^k] = \mu \sum_{j=1}^{i}(I-\mu H)^{i-j} \Ex \big[s_{\bsigma^k(i)}(\w_{0}^k) | \w_0^k\big] =0\nn
}
Computing the squared norm and taking expectations we get:
\eq{
	\Ex\|\widetilde{\w}_{i}^{\prime k}\|^2 = \Ex\|(I-\mu H)^i\widetilde{\w}_{0}^{\prime k}\|^2 + \mu^2\Ex\|s_i'(\w_0^k)\|^2
} 
We assume $\mu$ is sufficiently small so that $\|I-\mu H\| \leq 1-\mu\nu$, i.e. requiring $\mu \leq \frac{2}{\nu+\delta}$ and let  $t = (1-\mu\nu)^i$. Then, 
\eq{
	\Ex\|\widetilde{\w}_{i}^{\prime k}\|^2 \leq(1-\mu\nu)^{2i}\Ex\|\widetilde{\w}_{0}^{\prime k}\|^2 + \mu^2 \Ex\|s_i'(\w_0^k)\|^2
}
From Lemma \ref{lemma.2}, we know that 
\eq{
	R^{\prime\star}_{s,i}\define&\Ex s_i'(w^\star) s_i'(w^\star)\tran \nn\\
	=&\; \frac{N\left(\sum_{j=0}^{i-1} (I-\mu H)^jR_s^\star(I-\mu H)^j\right)}{N-1} \\
	&\;\; {}- \frac{\big[\sum_{j=0}^{i-1}(I-\mu H)^j\big] R_s^\star \big[\sum_{j=0}^{i-1}(I-\mu H)^j\big]}{N-1} \nn
}
and 
\eq{
	\Ex\|s_i'(w^\star)\|^2 = \Tr\big(\Ex s_i'(w^\star) s_i'(w^\star)\tran\big)
} 
With  $k\to\infty$, we obtain:\vspace{0.5mm}
\eq{
	\lim_{k\to\infty}\Ex\|\widetilde{\w}_{i}^{\prime k}\|^2\leq&
	(1-\mu\nu)^{2i}{\rm MSD}_{\rm RR}^{\rm lt} + \mu^2 \Tr\big(R^{\prime\star}_{s,i}\big)	 + O(\mu^4) \label{89n.sdg}
}
where the $O(\mu^4)$ term comes from the same argument in \eqref{89jgws.ge}. 

Substituting the result \eqref{msd.lt}, we have
\eq{
	\lim_{k\to\infty}\Ex\|\widetilde{\w}_{i}^{\prime k}\|^2\leq&(1-\mu\nu)^{2i}\mu^2\Tr\left(\big(I-(I-\mu H)^{2N}\big)^{-1}  R_{s}^{\prime \star}\right)\nn\\
	&	 +\mu^2\Tr\left( R_{s,i}^{\prime \star}\right) + O(\mu^4)\label{4h3dx}
	}
Lastly, we multiple $\Big(1\!-\!(1\!-\!\mu\nu)^{2i}\Big)$ and its inverse at the second term of \eqref{4h3dx}, which results in \eqref{2389g.cs}.
}

\section{Mismatch of Gradient Noise in (\ref{89jgws.ge})}\label{g892.gf3}\vspace{-2mm}
In this appendix, we will show that
\eq{
&\hspace{-3mm}\lim_{k\to \infty}\mu^2 \Tr\left(U\Sigma^\star U\tran \Ex s'(\w_0^k)s'(\w_0^k)\tran - \Ex s'(w^\star) s'(w^\star)\tran\right)\nn\\
 &= O(\mu^4)
}
which is equivalent to showing
\eq{
&\hspace{-3mm}\lim_{k\to \infty}\Tr\left(U\Sigma^\star U\tran \Ex s'(\w_0^k)s'(\w_0^k)\tran - \Ex s'(w^\star) s'(w^\star)\tran\right)\nn\\
&= O(\mu^2)
}
Using the inequality that $|\Tr(X)| \leq c\|X\|$ for any square matrix and some constant $c$,
we can just focus on the norm instead of trace:
\eq{
	&\hspace{-4mm}\left\|U\Sigma^\star U\tran \big(\Ex s'(\w_0^k)s'(\w_0^k)\tran - \Ex s'(w^\star) s'(w^\star)\tran\big)\right\|\nn\\
	&\leq \|U\Sigma^\star U\tran\| \left\|\Ex s'(\w_0^k)s'(\w_0^k)\tran - \Ex s'(w^\star) s'(w^\star)\tran\right\|
	\nn\\
	&= O(1/\mu)\left\|\Ex s'(\w_0^k)s'(\w_0^k)\tran - \Ex s'(w^\star) s'(w^\star)\tran\right\|
}
where the last equality is due to 
\eq{
	\|U\Sigma^\star U\tran\|
	=&\left\|\big(I -(I-\mu \Lambda)^{2N}\big)^{-1}\right\|\nn\\
	=&\left\|\big(2N\mu \Lambda + O(\mu^2)\big)^{-1}\right\|\nn\\
	=& O(1/\mu)
}
This result implies that we now need to show
\eq{
\lim_{k\to \infty}\left\|\Ex s'(\w_0^k)s'(\w_0^k)\tran - \Ex s'(w^\star) s'(w^\star)\tran\right\| = O(\mu^3)\nn
}

Since we have already established an expression for the covariance matrix of the gradient noise in \eqref{agg.noise.result} we have:
\begin{align}
&\hspace{-5mm}\Ex [s'(\w_0^k)s'(\w_0^k)\tran\,|\,\w_0^k]\nn\\
=&\; \frac{N\left(\sum_{i=0}^{N-1} (I-\mu H)^iR_s^k(I-\mu H)^i\right)}{N-1}\nn\\
&\;\; {}- \frac{\big[\sum_{i=0}^{N-1}(I-\mu H)^i\big] R_s^k \big[\sum_{i=0}^{N-1}(I-\mu H)^i\big]}{N-1}
\end{align}
Thus,
\eq{
	&\hspace{-7mm}\Ex [s'(\w_0^k)s'(\w_0^k)\tran\,|\,\w_0^k] - \Ex s'(w^\star) s'(w^\star)\tran
	\nn\\
	=&\, \frac{N\left(\sum_{i=0}^{N-1} (I-\mu H)^i\widetilde{R}_s^k(I-\mu H)^i\right)}{N-1}\nn\\
	&\; {}- \frac{\big[\sum_{i=0}^{N-1}(I-\mu H)^i\big] \widetilde{R}_s^k \big[\sum_{i=0}^{N-1}(I-\mu H)^i\big]}{N-1}\label{f32h78fwf}
}
where
\eq{
\widetilde{R}_s^k \define& R_s^k - R_s^\star \label{Rs_tilde}\\
R_s^k \define&\frac{1}{N} \sum_{n=1}^N \s_n(\w_0^k) \s_n(\w_0^k)\tran \label{Rs_k}\\
R_s^\star \define&\frac{1}{N} \sum_{n=1}^N \s_n(w^\star) \s_n(w^\star)\tran \label{Rs_star}
}
To simplify the notation, we rewrite the first term as follows:
\eq{
	&\hspace{-7mm}N\left(\sum_{i=0}^{N-1} (I-\mu H)^i\widetilde{R}_s^k(I-\mu H)^i\right)\nn\\
	=&\; \sum_{i=0}^{N-1}\sum_{j=0}^{N-1}(I-\mu H)^i\widetilde{R}_s^k(I-\mu H)^i\label{r32h89.g32}
}
Similarly, the second term:\vspace{-1mm}
\eq{
	&\hspace{-7mm}\left[\sum_{i=0}^{N-1}(I-\mu H)^i\right] \widetilde{R}_s^k \left[\sum_{i=0}^{N-1}(I-\mu H)^i\right]\nn\\[-1mm]
	=&\; \sum_{i=0}^{N-1}\sum_{j=0}^{N-1}(I-\mu H)^i\widetilde{R}_s^k(I-\mu H)^j\label{r2h389.ghj8}
}
Subtracting (\ref{r32h89.g32}) from (\ref{r2h389.ghj8}) we obtain (in the following, the notation $O(\mu^m)$ is a matrix where each entry can be bounded by $O(\mu^m)$):
\eq{
&\hspace{-6mm}
\Ex [s'(\w_0^k)s'(\w_0^k)\tran\,|\,\w_0^k] - \Ex s'(w^\star) s'(w^\star)\tran\nn\\
=&\;\frac{1}{N-1}\sum_{i=0}^{N-1}\sum_{j=0}^{N-1}(I-\mu H)^i\widetilde{R}_s^k[(I-\mu H)^i - (I-\mu H)^j]\nn\\
\stackrel{(a)}{=}&\;\frac{1}{N-1}\sum_{i=0}^{N-1}\sum_{j=0}^{N-1}(I-\mu H)^i\widetilde{R}_s^k[\mu (j-i)H + O(\mu^2) ]
\nn\\
\stackrel{(b)}{=}&\;\frac{1}{N-1}\mu \sum_{i=0}^{N-1}\sum_{j=0}^{N-1}(I-\mu H)^i\widetilde{R}_s^k (j-i)H + \widetilde{R}_s^kO(\mu^2)
\nn\\
\stackrel{(c)}{=} &\;\frac{1}{N-1} \mu \sum_{i=0}^{N-1}\sum_{j=0}^{N-1}\big(I+O(\mu)\big)\widetilde{R}_s^k (j-i)H + \widetilde{R}_s^kO(\mu^2)\nn\\
=&\;\frac{1}{N-1}\mu \underbrace{\sum_{i=0}^{N-1}\sum_{j=0}^{N-1}\widetilde{R}_s^k (j-i)}_{=0}H +
O(\mu^2)\widetilde{R}_s^kH+ \widetilde{R}_s^kO(\mu^2)\nn\\
=&\;O(\mu^2)\widetilde{R}_s^k H+ \widetilde{R}_s^kO(\mu^2) \label{r31289.g1}
}
where steps (a) and (c) use the binomial expansion, and step (b) assumes the step-size is small enough so that $I-\mu H$ is stable.
Next, we conclude:
\eq{
	&\hspace{-4mm}\left\|\Ex s'(\w_0^k)s'(\w_0^k)\tran - \Ex s'(w^\star) s'(w^\star)\tran\right\|\nn\\
	=\,&\left\|\Ex_{\w^k_0}\Big[ \Ex s'(\w_0^k)s'(\w_0^k)\tran\,|\, \w_0^k\Big] - \Ex s'(w^\star) s'(w^\star)\tran \right\|
	\nn\\
	\stackrel{(a)}{\leq}\,&\,\Ex_{\w^k_0}\left\|\Big[ \Ex s'(\w_0^k)s'(\w_0^k)\tran\,|\, \w_0^k\Big] - \Ex s'(w^\star) s'(w^\star)\tran \right\|
	\nn\\
	\stackrel{(\ref{r31289.g1})}{=}&\Ex_{\w^k_0}\|O(\mu^2)\widetilde{R}_s^k H+ \widetilde{R}_s^kO(\mu^2)\|
	\nn\\
	\leq\,&\, O(\mu^2)\Ex\| \widetilde{R}_s^k\|
}
where step (a) applies Jensen's inequality. Lastly, we prove
\be
	\lim_{k\to\infty} \Ex\| \widetilde{R}_s^k\| = O(\mu)
\ee
From \eqref{Rs_tilde}-\eqref{Rs_star}, we have
\eq{\label{28hanelkcuys}
\widetilde{R}_s^k =& R_s^k - R_s^\star \nnb
=&\;\frac{1}{N}\sum_{n=1}^{N}\left[ s_n(\w_0^k) s_n(\w_0^k)\tran - s_n(w^\star) s_n(w^\star)\tran \right] \nnb
=&\; \frac{1}{N}\sum_{n=1}^{N}\left[ s_n(\w_0^k) [s_n(\w_0^k)-s_n(w^\star)]\tran\right.
\nnb&\;\hspace{12mm}
\left.{}+[s_n(\w_0^k)- s_n(w^\star)] s_n(w^\star)\tran\right]
}
Next, it is easy to verify that $s_{n}(w)$ is also $2\delta$-Lipschitz continuity:
\eq{
	&\hspace{-6mm}\|s_{n}(\w_{0}^k)- s_{n}(w^\star)\|
	\nn\\
	\stackrel{}{\leq}&\; \|\nabla J(\w_{0}^k) - \nabla J(w^\star)\|
	 + \|\nabla Q(\w_{0}^k; x_{n})-\nabla Q(w^\star; x_{n}) \|
	\nn\\
	\stackrel{(\ref{eq-ass-cost-lc-e})}{\leq}&\; 2\delta \|\widetilde{\w}^k_0\|\label{noise.lipschitz}
}
Taking the expectation of the norm of \eqref{28hanelkcuys}:
\eq{
\Ex\|\widetilde{R}_s^k\|\leq&\, \frac{1}{N}\sum_{n=1}^{N}\Ex\left\| s_n(\w_0^k) [s_n(\w_0^k)-s_n(w^\star)]\tran\right.
\nnb&\;\hspace{12mm}
\left.{}+[s_n(\w_0^k)- s_n(w^\star)] s_n(w^\star)\tran\right\|
\nnb
\stackrel{\eqref{noise.lipschitz}}{\leq}&\;\frac{2}{N}\sum_{n=1}^{N}\Ex \left(\|s_n(\w_0^k) \| \delta\|\widetilde{\w}^k_0\| + \delta\|\widetilde{\w}^k_0\| \|s_n(w^\star)\|\right)
\nn\\
\leq&\frac{2\delta}{N}\!\sum_{n=1}^{N} \!\sqrt{\Ex\! \|s_n(\w_0^k) \|^2\Ex\|\widetilde{\w}^k_0\|^2} \!+\! \sqrt{\Ex\!\|\widetilde{\w}^k_0\|^2} \|s_n(w^\star)\|\label{389jg4.3}
}
where the last inequality exploits the Cauchy-Schwartz inequality.
Next, as we prove in theorem \ref{lemma.start.pont}, when $k\gg 1$:
\eq{
	\Ex\|\widetilde{\w}^k_0\|^2 \,=&\; O(\mu^2)\\
	\Ex \|s_n(\w_0^k) \|^2 \leq\,&\;2\Ex \|s_n(\w_0^k) -s_n(w^\star)\|^2 +2\Ex \|s_n(w^\star)\|^2\nn\\
	\leq& O(\mu^2)+O(1) = O(1)
}
Substituting the previous results into (\ref{389jg4.3}), we conclude
\eq{
\Ex\|\widetilde{R}_s^k\|\leq&\,	\frac{\delta}{N}\sum_{n=1}^{N}\left(\sqrt{O(\mu^2)O(1)} + \sqrt{O(\mu^2)}O(1)\right)
\nn\\
=&\;O(\mu), \;\; k\gg1
}\vspace{-3mm}

\section{Bound on long-term difference}\label{mismatch.proof}\vspace{-2mm}
Subtracting (\ref{eq.error_expand}) from (\ref{eq.long_term.rec}) and then taking the conditional expectation, we obtain:
\begin{align}
&\hspace{-3mm}\Ex\big[\,\|\widetilde{\w}_0^{k+1} - \widetilde{\w}_0^{\prime k+1}\|^2 \,|\,\widetilde{\w}_0^{k} ,  \widetilde{\w}_0^{\prime k}\,\big]
\nn\\
\leq&\; \frac{1}{t}\|(I-\mu H)^N\| \|\widetilde{\w}_{0}^k -\widetilde{\w}_{0}^{\prime k}\|^2\nn\\
& +\frac{2\mu^2}{1-t}\! \underbrace{\Ex \!\left\|\sum^N_{i=1} ( I \!-\! \mu H )^{N-i} \left(s_{\bsigma^k(i)}(\w_{i-1}^k)\!-\! s_{\bsigma^k(i)}(\w_{0}^k)\right)\right\|^2}_{B}
\nn\\
& +  \frac{2\mu^2}{1-t} \underbrace{\Ex\left\|\sum^N_{i=1} ( I - \mu H )^{N-i} \xi(\w_{i-1}^k)\right\|^2}_{C}\label{h89gh92}
\end{align}
where we exploit the Jensen's inequality and $0<t<1$.
In the following, we assume the step size is sufficiently small so that:
\eq{
	\|I-\mu H\| \leq 1-\mu \nu
}
\noindent Now, we find a tighter bound on the $B$ term:
\eq{
B\stackrel{(a)}{\leq}&\Ex \!\left(\!\sum^N_{i=1}\left\|\! ( I \!-\! \mu H )^{N-i}\right\|\! \left\|s_{\bsigma^k(i)}(\w_{i-1}^k)\!-\! s_{\bsigma^k(i)}(\w_{0}^k)\right\|\!\right)^2
\nn\\
\stackrel{\eqref{noise.lipschitz}}{\leq}&\,\Ex  \left(\sum^N_{i=1}\left\| ( I - \mu H )^{N-i}\right\| 2\delta\mu\left\| \w_{i-1}^k - \w_{0}^k\right\|\right)^2
\nn\\
=&\,4\delta^2\mu^2\sum_{i=1}^N \sum_{j=1}^N \|I-\mu H\|^{(N-i)(N-j)}\times
\nn\\
&\;\; \Ex\!\left\| \!\sum_{n=1}^{i-1}\grad Q(\w_{n-1}^k;x_{\bsigma^k(n)})\right\|\!
\left\| \sum_{n=1}^{j-1}\grad\! Q(\w_{n-1}^k;x_{\bsigma^k(n)})\right\|
\nn\\
\stackrel{(b)}{\leq}&\,4\delta^2\mu^2\sum_{i=1}^N \sum_{j=1}^N (1-\mu\nu)^{(N-i)(N-j)}\times
\nn\\
& \!\sqrt{\!\Ex\!\left\|\! \sum_{n=1}^{i-1}\hspace{-0.5mm}\grad\hspace{-0.5mm} Q(\w_{n-1}^k;x_{\bsigma^k(n)})\right\|^2 \!\hspace{-1.2mm}\Ex\!
\left\|\! \sum_{n=1}^{j-1}\hspace{-0.5mm}\grad\hspace{-0.5mm} Q(\w_{n-1}^k;x_{\bsigma^k(n)})\right\|^2}
\nn\\
=\,&4\delta^2\mu^2\hspace{-1mm}\left(\hspace{-0.7mm}\sum_{i=1}^N (1\!-\!\mu\nu)^{N-i}\hspace{-1.2mm}
\sqrt{\Ex\Bigg\| \sum_{n=1}^{i-1}\grad Q\big(\w^k_{n-1};x_{\bsigma^k(n)}\big)\Bigg\|^2}\right)^2
\nn\\
\label{24h89f3.g32}
}
where step (a) exploits the triangular inequality, and the sub-multiplicative property of norms, and step (b) uses Cauchy-Schwartz.
Then, we establish the following when $k$ is large enough:
\eq{
&\hspace{-5mm}\Ex\Big\| \sum_{n=1}^{i-1}\grad_w Q\big(\w^k_{n-1};x_{\bsigma^k(n)}\big)\Big\|^2 \nn\\
=&\,\Ex\Bigg\| \sum_{n=1}^{i-1}\Big(\grad Q(\w_0^k;x_{\bsigma^k(n)}) - \grad Q(w^\star;x_{\bsigma^k(n)}) +\nn\\
&\;\;\;\;\grad Q(w^\star;x_{\bsigma^k(n)}) \Big)\Bigg\|^2
\nn\\
\stackrel{}{\leq}&\,2\Ex\Bigg\| \sum_{n=1}^{i-1}\grad_w Q\big(w^\star;x_{\bsigma^k(n)}\big)\Bigg\|^2 \nn\\
&\;\;+2\Ex\Bigg\| \sum_{n=1}^{i-1}\left(\grad Q(\w_0^k;x_{\bsigma^k(n)}) - \grad Q(w^\star;x_{\bsigma^k(n)})
\right)\Bigg\|^2 
\nn\\
=&\,2\frac{(i-1)N-(i-1)^2}{N-1}\cK + O(\mu^2)
}
where the last equality is because we already conclude from Lemma \ref{lemma.2} and \eqref{gradient.noise} that
\eq{
	\Ex\left\| \sum_{n=1}^{i-1}\hspace{-0.5mm}\grad_w Q(w^\star;x_{\bsigma^k(n)})\right\|^2 =& \frac{(i-1)N-(i-1)^2}{N-1}\cK
}
Moreover, we know that for sufficiently large $k$:
\eq{
	&\hspace{-5mm}\Ex\left\| \sum_{n=1}^{i-1}\left(\grad_w Q(\w_0^k;x_{\bsigma^k(n)}) - \grad_w Q(w^\star;x_{\bsigma^k(n)}) \right)\right\|^2
	\nn\\
	\leq&\,(i-1)\Ex\sum_{n=1}^{i-1}\left\| \grad_w Q(\w_0^k;x_{\bsigma^k(n)}) - \grad_w Q(w^\star;x_{\bsigma^k(n)})\right\|^2
	\nn\\
	\leq&\, \delta^2 (i-1)\Ex\sum_{n=1}^{i-1}\|\widetilde{\w}_{i-1}^k\|^2
	\nn\\
	=&\,O(\mu^2)
}
Substituting previous results into (\ref{24h89f3.g32}):
\eq{
B
\leq&4\delta^2\mu^2\Bigg(\sum_{i=1}^N  (1-\mu\nu)^{N-i}\times\\
&\hspace{15mm}
\sqrt{2\frac{(i-1)N-(i-1)^2}{N-1}+O(\mu^2)}\Bigg)^2\!\cK\nn
}
We know for any $0\leq i\leq N$
\eq{
	\frac{(i-1)N-(i-1)^2}{N-1}\leq \frac{N^2}{4(N-1)}
}
and, hence, 
\eq{
	B\leq& 4\delta^2\mu^2\left(\frac{N^2}{2(N-1)} + O(\mu^2)\right) \left(\frac{1-(1-\mu\nu)^N}{\mu\nu}\right)^2\cK\nn\\
	=&\frac{2\delta^2N^2}{\nu^2(N-1)}(1-(1-\mu\nu)^N)^2\cK +O(\mu^2)\label{bound.B}
}
We can bound the term $C$ when epoch $k$ is sufficiently large:
\eq{
	C \leq& N \sum_{i=1}^N\Ex\|(I-\mu N)^{N-i} \xi(\w^k_{i-1})\|^2\nn\\
	\stackrel{\eqref{3h98u.ni}}{\leq}& \frac{\kappa^2N}{4} \sum_{i=1}^N \Ex\|\tw_{i-1}^k\|^4\nn\\
	=&O(\mu^4) \label{gi.d}
}
where the last equality is due to \eqref{xcnweu-2}:
\eq{
	\|\tw_0^{k+1}\|^4 \leq\;& \left(\Big(1-\frac{1}{2}\mu N\nu\Big)\|\tw^k_0\|^2+\frac{2\mu^3\delta^2N^3}{\nu}\mathcal{K}\right)^2\nn\\
	\leq\;& \frac{\Big(1-\frac{1}{2}\mu N\nu\Big)^2}{s}\|\tw^k_0\|^4+\frac{4\mu^6\delta^4N^6}{(1-s)\nu^2}\mathcal{K}^2
}
Let $s=1-\frac{1}{2}\mu N\nu$, we obtain:
\eq{
	\|\tw_0^{k+1}\|^4 \leq\;&\Big(1-\frac{1}{2}\mu N\nu\Big)\|\tw^k_0\|^4 + \frac{8\mu^5\delta^4N^5}{\nu^3}\mathcal{K}^2
}
After letting $k\to\infty$ and taking expectation, we conclude $\Ex \|\widetilde{\w}_0^k\|^4=O(\mu^4)$.

Lastly, choosing $t=(1-\mu\nu)^N$ in \eqref{h89gh92} and combining \eqref{bound.B} and \eqref{gi.d}, we establish:
\eq{
	&\hspace{-4mm}\Ex\big[\,\|\widetilde{\w}_0^{k+1} - \widetilde{\w}_0^{\prime k+1}\|^2\big]\nn\\ \leq&\, (1-\mu\nu)^N\Ex\|\widetilde{\w}_0^{k} - \widetilde{\w}_0^{\prime k}\|^2\nn\\
	&\;+\frac{2\mu^2}{1-(1-\mu\nu)^N}\frac{2\delta^2N^2}{\nu^2(N-1)}(1-(1-\mu\nu)^N)^2\cK + O(\mu^4)
}
Letting $k\to \infty$, we conclude
\eq{
	\Ex\|\widetilde{\w}_0^{k} - \widetilde{\w}_0^{\prime k}\|^2\leq \frac{4\mu^2\delta^2N^2}{\nu^2(N-1)}\cK +O(\mu^3)
}

\smallskip
\bibliographystyle{IEEEbib}
\bibliography{random_reshuffle}

\end{document}